%% file: xyz.tex
\documentclass[12pt]{article}

\usepackage{amsmath}
\usepackage{amsfonts}
\usepackage{amsthm} 
\usepackage{amssymb}
\usepackage{bm}
\usepackage{graphicx}
\usepackage{dsfont}
\usepackage{natbib}
\usepackage{xcolor}
\usepackage{graphicx}
\usepackage{subfig}
\usepackage{bbm,mathrsfs}
\usepackage{centernot}
\usepackage{authblk}

\usepackage{fullpage, enumerate}

\usepackage{algorithm}
\usepackage{algpseudocode}
\usepackage{pifont}

\newtheorem{thm}{Theorem}
\newtheorem{lem}[thm]{Lemma}
\newtheorem{prop}[thm]{Proposition}
\newtheorem{cor}[thm]{Corollary}

\newcommand{\abs}[1]{\left| #1\right|}

\newcommand{\ind}{\mathbbm{1}}
\newcommand{\pr}{\mathbb{P}}
\newcommand{\R}{\mathbb{R}}
\newcommand{\E}{\mathbb{E}}
\newcommand{\N}{\mathbb{N}}
\newcommand{\floor}[1]{\left\lfloor #1 \right\rfloor}
\newcommand{\ceil}[1]{\left\lceil #1 \right\rceil}

\newcommand{\argmin}[1]{\underset{#1}{\operatorname{arg}\operatorname{min}}\;}
\newcommand{\sgn}{\mathrm{sgn}}

\newcommand{\eqdist}{\stackrel{d}{=}}

\newcommand{\mb}{\mathbf}
\newcommand{\mbb}{\boldsymbol}
\newcommand{\matr}[1]{\mathbf{#1}}
\title{The \emph{xyz} algorithm for fast interaction search in high-dimensional data}

\author[1]{Gian-Andrea Thanei}
\author[1]{Nicolai Meinshausen}
\author[2]{Rajen D.\ Shah\thanks{Supported by the Isaac Newton Trust Early Career Support Scheme, the Alan Turing Institute under the EPSRC grant EP/N510129/1 and EPSRC Programme Grant EP/N031938/1}}
\affil[1]{ETH Z\"urich}
\affil[2]{University of Cambridge}

\begin{document}
\maketitle
\input{abstract}

\input{main_text}

\end{document}

%% file: abstract.tex
\begin{abstract}
When performing regression on a data set with $p$ variables, it is often of interest to go beyond using main linear effects and include interactions as products between individual variables. For small-scale problems, these interactions can be computed explicitly but this leads to a computational complexity of at least $\mathcal{O}(p^2)$ if done naively. This cost can be prohibitive if $p$ is very large.

We introduce a new randomised algorithm that is able to discover interactions with high probability and under mild conditions has a runtime that is subquadratic in $p$. We show that strong interactions can be discovered in almost linear time, whilst finding weaker interactions requires $\mathcal{O}(p^\alpha)$ operations for $1<\alpha<2$ depending on their strength. The underlying idea is to transform interaction search into a closest pair problem which can be solved efficiently in subquadratic time. The algorithm is called \emph{xyz} and is implemented in the language \texttt{R}. We demonstrate its efficiency for application to genome-wide association studies, where more than $10^{11}$ interactions can be screened in under $280$ seconds with a single-core $1.2$ GHz CPU.\end{abstract}

%% file: main_text.tex
\section{Introduction}

Given a response vector $\mb Y \in \mathbb{R}^n$ and matrix of associated predictors $\matr{X}=(\matr{X}_1,\ldots,\matr{X}_p) \in \mathbb{R}^{n \times p}$, finding interactions is often of great interest as they may reveal important relationships and improve predictive power. When the number of variables $p$ is large, fitting a model involving interactions can involve serious computational challenges. The simplest form of interaction search consists of screening for pairs $(j,k)$ with high inner product between the outcome of interest $\mb Y$ and the point-wise product $\matr{X}_j \circ \matr{X}_k$:
\begin{equation} \label{eq:inter_search1}
\text{Keep all pairs } (j,k) \text{ for which } \mb Y^T(\matr{X}_j \circ \matr{X}_k)/n > \kappa.
\end{equation}
This search  is of complexity  $\mathcal{O}(np^2)$ in a naive implementation and quickly becomes infeasible for large $p$.
Of course one would typically be interested in maximising (absolute values of) correlations rather than dot products in \eqref{eq:inter_search1}, an optimisation problem that would be at least as computationally intensive.

Even more challenging is the task of fitting a linear regression model involving pairwise interactions:
\begin{equation} \label{eq:inter_mod}
Y_i = \mu + \sum_{j=1}^p X_{ij} \beta_j + \sum_{k=1}^p \sum_{j=1}^{k-1} X_{ij} X_{ik} \matr{\theta}_{jk}  +\varepsilon_i.
\end{equation}
Here $\mu \in \mathbb{R}$ is the intercept and $\beta_j$ and $\theta_{jk}$ contain coefficients for main effects and interactions respectively, and $\varepsilon_i$ is random noise.

In this paper, we make several contributions to the problem of searching for interactions in high-dimensional settings.
\begin{enumerate}[(a)]
\item We first establish a form of equivalence between \eqref{eq:inter_search1} and closest-pair problems \citep{Shamos1975,Agarwal1991}.
Assume for now that all predictors and outcomes are binary, so  $X_{ij}, Y_i \in \{-1,1\}$ (we will later relax this assumption) and define  $\matr{Z} \in \{-1,1\}^{n \times p}$ as $Z_{ij} = Y_i X_{ij}$. Then it is straightforward to show that \eqref{eq:inter_search1} is  equivalent to
\begin{equation} \label{eq:close_pairs1}
\text{Keep all pairs } (j,k) \text{ for which } \|\matr{X}_j - \matr{Z}_k\|_2 < \kappa'
\end{equation}
for some $\kappa'$.
This connects the search for interactions to literature in computational geometry
%\citep{Shamos1975,Agarwal1991}
on problems of finding closest pairs of points.
\item We introduce the \emph{xyz} algorithm to solve \eqref{eq:close_pairs1} based on randomly projecting each of the columns in $\matr{X}$ and $\matr{Z}$ to a one-dimensional space. By exploiting the ability to sort the resulting $2p$ points with $\mathcal{O}(p \log(p))$ computational cost, we achieve a run time that is always subquadratic in $p$ and can  even reach a linear complexity $\mathcal{O}(np)$ when $\kappa$ is much larger than the quantity $|\mb Y^T(\matr{X}_j \circ \matr{X}_k)|/n$ of the bulk of the pairs $(j,k)$. We show that our approach can be viewed as an example of locality sensitive hashing \citep{Leskovec2014} optimised for our specific problem.
\item We show how any method for solving \eqref{eq:inter_search1} can be used to fit regression models with interactions \eqref{eq:inter_mod} by building it into an algorithm for the Lasso \citep{tibshirani96regression}. The use of \emph{xyz} thus leads to a procedure for applying the Lasso to all main effects and interactions with computational cost that scales subquadratically in $p$.
\item We provide implementations of both the core \emph{xyz} algorithm and its extension to the Lasso in the R package \texttt{xyz}, which is available on github~\citep{githublink} and CRAN.%(Package: \textit{xyz}).
\end{enumerate}
Our work here is thus related to ``closest pairs of points'' algorithms in computational geometry as well as an extensive literature on modelling interactions in statistics, both of which we now review.

\subsection{Related work} \label{sec:related_work}
A common approach to avoid the quadratic cost in $p$ of searching over all pairs of variables~\eqref{eq:inter_search1} is to restrict the search space: one can first seek a small number of important main effects, and then only consider interactions involving these discovered main effects.
More specifically, one could fit a main effects Lasso \citep{tibshirani96regression} to the data first, add interactions between selected main effects to the matrix of predictors, and then run the Lasso once more on the augmented design matrix in order to produce the final model (see \citet{Wu_etal2010} for example). Tree-based methods such as CART \citep{CART} work in a similar fashion by aiming to identify an important main effect and then only considering interactions involving this discovered effect.

However it is quite possible for the signal to be such that main effects corresponding to important interactions are hard to detect. As a concrete example of this phenomenon, consider the setting where $\matr{X}$ is generated randomly with all entries independent and having the uniform distribution on $\{-1,1\}$. Suppose the response is given by $Y_i = X_{i1} X_{i2}$, so there is no noise. Since the distribution $Y_i | X_{ij}$ is the same for all $k$, main effects regressions would find it challenging to select variables 1 and 2. Note that by reparametrising the model by adding one to each entry of $\matr{X}$ for example, we obtain $Y_i = (X_{i1}-1)(X_{i2}-1) = 1 -X_{i1}- X_{i2} + X_{i1}X_{i2}$. The model now respects the so-called strong hierarchical principle \citep{Bien2013} that interactions are only present when their main effects are. The hierarchical principle is useful to impose on any fitted model. However, imposing the principle on the model does not imply that the interactions will easily be found by searching for main effects first. The difficulty of the example problem is due to interaction effects masking main effects: this is a property of the signal $\E(Y_i)$ and of course no reparametrisation can make the main effects any easier to find.
Approaches that increase the set of interactions to be considered iteratively can help to tackle this sort of issue in practice \citep{Bickel_etal2010, Hao2014, MARS, Shah2016} as can those that randomise the search procedure \citep{breiman01random}. However they cannot eliminate the problem of missing interactions, nor do these approaches offer guarantees of how likely it is that they discover an interaction.

As alluded to earlier, the pure interaction search problem \eqref{eq:close_pairs1} is related to close pairs of points problems, and more specifically the close bichromatic pairs problem in computational geometry \citep{Agarwal1991}. Most research in this area has focused on algorithms that lead to computationally optimal results in the number of points $p$ whilst considering the dimension $n$ to be constant. This has resulted in algorithms where the scaling of the computational complexity with $n$ is at least of order $2^n$ \citep{Shamos1975}. Since for meaningful statistical results one would typically require $n \gg \log(p)$, these approaches would not lead to subquadratic complexity.
An exception is the so-called lightbulb algorithm \citep{Paturi1989} which employs a similar strategy for binary data; our work here shows that this is optimal among random projection-based methods and also that it may be modified to handle continuous data and also detect interactions in high-dimensional regression settings.

In the special case where $n = p$ and $Z_{ij}, X_{ij} \in \{-1, 1\}$, \eqref{eq:close_pairs1} may be seen to be equivalent to searching for large magnitude entries in the product of square matrices $\matr{X}$ and $\matr{Z}^T$. This latter problem is amenable to fast matrix multiplication algorithms, which in theory can deliver a subquadratic complexity of roughly $\mathcal{O}(p^{2.4})=\mathcal{O}(np^{1.4})$ \citep{Williams2012, Davie2013, LeGall2012}. However the constants hidden in the order notation are typically very large, and practical implementations are unavailable. The Strassen algorithm \citep{Strassen1969} is the only fast matrix multiplication algorithm used regularly in practice to the best of our knowledge. With a complexity of roughly $\mathcal{O}(p^{2.8})=\mathcal{O}(np^{1.8})$, the improvement over a brute force close pairs search is only slight.

%\gian{Strategies that give a fixed improvement of run time (not depending on the signal strength) but are not restricted to the case $n=p$ are summarized in the literature as the light bulb problem.}

The strategy we use is most closely related to locality sensitive hashing  (LSH) \citep{indyk1998approximate} which encompasses a family of hashing procedures such that similar points are mapped to the same bucket with high probability. A close pair search can then be conducted by searching among pairs mapped to the same bucket. In fact, our approach for solving \eqref{eq:close_pairs1} can be thought of as an example of LSH optimised for our particular problem setting. This connection is detailed in Appendix B.

A seemingly attractive alternative to the subsampling-based LSH-strategy we employ is the method of random projections which is motivated by the theoretical guarantees offered by the Johnson--Lindenstrauss Lemma \citep{Achlioptas2003}. Perhaps surprisingly, we can show that using random projections instead of  our subsampling-based scheme  leads to a quadratic run time for interaction search (see Theorem \ref{thm:optim} and section \ref{sec:gaussexp}).

An approach that bears some similarity with our procedure is that of \emph{epiq} \citep{Arkin2014}. This works by projecting the data and then searches through a lower dimensional representation for close pairs. This appears to improve upon a naive brute force empirically but there are no proven guarantees that the run time improves on the  $\mathcal{O}(np^2)$ complexity of a naive search.

The \emph{Random Intersection Trees} algorithm of \citet{Shah2014} searches for potentially deeper interactions in data with both $\matr{X
}$ and $\mb Y$ binary. In certain cases with strong interactions a complexity close to linear in $p$ is achieved; however it is not clear how to generalise the approach to continuous data or embed it within a regression procedure.

%\gian{In \citet{tyagi2017} the authors consider more general additive models with interaction terms. They propose an efficient compressed sensing scheme for estimating the sparse Hessian, from which the set of important interactions is deduced. They assume that they can query $\mb Y=f( \mb X)$ as often as needed to get sufficient accuracy for the Gradient and the Hessian. In our setting we do not have the ability to query $\mb Y$ for new values of $\mb X$ and hence derivatives are not available to us.}

The idea of \citet{kong2016distcor} is to first transform the data by forming $\tilde{\mb Y} = \mb Y \circ \mb Y$ and $\tilde{\mb X}_j = \mb X_j \circ \mb X_j$ for each predictor. Next $\tilde{\mb X}_j$ and $\tilde{\mb Y}$ are tested for independence using the distance correlation test. In certain settings, this can reveal important interactions with a computational cost linear in $p$. However, the powers of these tests depend on the distributions of the transformed variables $\tilde{\matr{X}}_j$. For example in the binary case when $\matr{X} \in \{-1,1\}^{n \times p}$, each transformed variable will be a vector of 1's and the independence tests will be unhelpful. We will see that our proposed approach  works particularly well in this setting.

\subsection{Organisation of the paper}

In Section~\ref{sec:binary} we consider the case where both the response $\mb Y$ and the predictors $\mb X$ are binary. We first demonstrate how \eqref{eq:inter_mod} may be converted to a form of closest pair of points problem. We then introduce a general version of the \emph{xyz} algorithm which solves this based on random projections. As we show in Section~\ref{sec:optmin} there is a particular random projection distribution that is optimal for our purposes. This leads to our final version of the \emph{xyz} algorithm which we present in Section~\ref{sec:xyz_prop} along with an analysis of its run time and probabilistic guarantees that it recovers important interactions. In Section \ref{sec:cont} we extend the \emph{xyz} algorithm to continuous data. These ideas are then used in Section~\ref{sec:regression} to demonstrate how the \emph{xyz} algorithm can be embedded within common algorithms for high-dimensional regression \citep{friedman2010regularization} allowing high-dimensional regression models with interactions to be fitted with subquadratic complexity in $p$. Section~\ref{sec:experi} contains a variety of numerical experiments on real and simulated data that complement our theoretical results and demonstrate the effectiveness of our proposal in practice. We conclude with a brief discussion in Section~\ref{sec:discuss} and all proofs are collected in the Appendix.

\section{The \emph{xyz} algorithm for binary data} \label{sec:binary}

In this section, we present a version of the \emph{xyz} algorithm applicable in the special case where both $\matr{X}$ and $\matr{Y}$ are binary, so $X_{ij} \in \{-1, 1\}$ and $Y_i \in \{-1,1\}$. We build up to the algorithm in stages, giving the final version in Section~\ref{sec:final_xyz}.

Define $\matr{Z} \in \{-1,1\}^{n \times p}$ by $Z_{ij} = Y_iX_{ij}$ and
\begin{equation} \label{eq:def_gamma}
\gamma_{jk} = \frac{1}{n} \sum_{i=1}^n \ind_{\{Y_i = X_{ij}X_{ik}\}}.
\end{equation}
We call $\gamma_{jk}$ the interaction strength of the pair $(j,k)$. It is easy to see that the interaction search problem \eqref{eq:inter_search1} can be expressed in terms of either the $\gamma_{jk}$ or the normalised squared distances. Indeed
\begin{equation} \label{eq:equiv}
 2\gamma_{jk} -1 = \mb Y^T(\matr{X}_j \circ \matr{X}_k)/n = \matr{Z}_j^T \matr{X}_k/n = 1-\|\matr{Z}_j - \matr{X}_k\|_2^2/(2n).
\end{equation}
Thus those pairs $(j,k)$ with $\mb Y^T(\matr{X}_j \circ \matr{X}_k)/n$ large will have $\gamma_{jk}$ large, and $\|\matr{Z}_j - \matr{X}_k\|_2^2$ small.
This equivalence suggests that to solve \eqref{eq:inter_search1}, we can search for pairs $(j,k)$ of columns $\matr{Z}_j, \matr{X}_k$ that are close in $\ell_2$ distance. At first sight, this new problem would also appear to involve a search across all pairs, and would thus incur an $\mathcal{O}(np^2)$ cost.
As mentioned in the introduction, close pair searches that avoid a quadratic cost in $p$ incur typically an exponential cost in $n$. Since $n$ would typically be much larger than $\log(p)$, such searches would be computationally infeasible.

We can however project each of the $n$-dimensional columns of $ \matr{X}$ and $\mb Z$ to a lower dimensional space and then perform a close pairs search.
The Johnson--Lindenstrauss Lemma, which states roughly that one can project $p$ points into a space of dimension $\mathcal{O}(\log(p))$ and faithfully preserve distances, may appear particularly relevant here. The issue is that the projected dimension suggested by the Johnson--Lindenstrauss Lemma is still too large to allow for an efficient close pairs search.
The following observation however gives some encouragement: if we had $\mb Y=\matr{X}_j \circ \matr{X}_k$ so $\matr{X}_j = \matr{Z}_k$, even a one-dimensional projection $\mb R \in \R^n$ will have $|\mb R^T(\matr{X}_j-\matr{Z}_k)|=0=\|\matr{X}_j-\matr{Z}_k\|_2$, which implies that a perfect interaction will have zero distance in the projected space. We will later see that our approach leads to a linear run time in such a case. Importantly, we are only interested in using a projection that preserves the distances between the close pairs rather than all pairs, which makes our problem very different to the setting considered in the Johnson--Lindenstrauss Lemma.

With this in mind, consider the following general strategy. First project the columns of $\mb X$ and $\mb Z$ to one-dimensional vectors $x$ and $z$ using a random projection $\mb R$: $\mb x=\mb X^T \mb R$, $\mb z=\mb Z^T \mb R$. Next for some threshold $\tau$, collect all pairs $(j,k)$ such that $|x_j - z_k| \leq \tau$ in the set $E$. By first sorting $\mb x$ and $\mb z$, a step requiring only $\mathcal{O}(p\log(p))$ computations (see for example \citet{Sedgewick1998}), this close pairs search can be shown to be very efficient. Given this set of candidate interactions, we can check for each $(j,k)\in E$ whether we have $\gamma_{jk} \geq \gamma$.
%$\mb Y^T(\matr{X}_j \circ \matr{X}_k)/n > \gamma$ \gian{here}.
The process can be repeated $L$ times with different random projections, and one would hope that given enough repetitions, any given strong interaction would be present in one of the candidate sets $E_1,\ldots,E_L$ with high probability.
This approach is summarised in Algorithm~\ref{alg:general_xyz} which we term the general form of the \emph{xyz} algorithm. A schematic overview is given in Figure~\ref{fig:overview}.

\begin{algorithm}
\caption{\label{alg:general_xyz}A general form of the \emph{xyz} algorithm.}
\begin{algorithmic}[1]
\Statex \textbf{Input}: $\matr{X} \in \{-1,1\}^{n \times p}$, $\mb Y \in \{-1,1 \}^n$
\Statex \textbf{Parameters:} $\xi=(G,L,\tau,\gamma)$. Here $G$ is the joint distribution for the projection vector $\mb R$, $L$ is the number of projections, and $\tau$ and $\gamma$ are the thresholds for close pairs and interactions strength respectively.
\Statex \textbf{Output}: $I$ set of strong interactions.
\State Form $\mb Z$ via $Z_{ij}=Y_i X_{ij}$ and set $I:=\emptyset$.
\For{$l \in \{1,\ldots,L\}$ }
\State Draw random vector $\mb R \in \R^n$ with distribution $G$ and project the data using $\mb R$, to form \[ \mb x=\mb X^T \mb R \mbox{ and } \mb z=\mb Z^T \mb R.\]
\State Collect in $E_l$ all pairs $(j, k)$ such that $|x_j-z_k|\leq \tau$.
\State Add to $I$ those $(j,k) \in E_l$ for which $\gamma_{jk} \geq \gamma$.
%$(|\mb Z_k^T \mb X_j|/n + 1)/2 > \gamma$.
%$|\matr{X}_j ^T \matr{Z}_k|/n > \gamma$. \Rajen{$(|\mb Z_k^T \mb X_j|/n + 1)/2 > \gamma$}
\EndFor
\end{algorithmic}
\end{algorithm}
There are several parameters that must be selected, and a key choice to be made is the form of the random projection $\mb R$.
For the joint distribution $G$ of $\mb R$ we consider the following general class of distributions, which includes both dense and sparse projections. We sample a random or deterministic number $M$ of indices from the set $\{1,\ldots,n\}$, $i_1,\ldots, i_M$, either with or without replacement. Then, given a distribution $F \in \mathcal{F}$ where $\mathcal{F}$ is a class of distributions to be specified later, we form a vector $\mb D \in \R^M$ with independent components each distributed according to $F$. We then define the random projection vector $\mb R$ by
\begin{equation}
\label{eq:projR}
R_i = \sum_{m=1}^M D_m \mathds{1}_{\{ i_m =i\}}, \qquad i=1,\ldots,n.
\end{equation}

%Each configuration $\xi$ of the \emph{xyz} algorithm contains the following parameters:
Each configuration of the \emph{xyz} algorithm is characterised by fixing the following parameters:
\begin{enumerate}[(i)]
\item $G$, %an exchangeable distribution
a distribution
for the projection vector $R$ which is determined through \eqref{eq:projR} by $F \in \mathcal{F}$, a distribution for the subsample size $M$ and whether sampling is with replacement or not;
%The choice of $G$ entails a distribution $F \in \mathcal{F}$ and in case of subsampling a possibly random subsample size $M$.
\item $L \in \mathbb{N}$, the number of projection steps;
\item $\tau \geq 0$, the close pairs threshold;
\item $\gamma \in (0,1)$, the interaction strength threshold.
\end{enumerate}
We will denote the collection of all possible parameter levels by $\Xi$.
%All algorithms are collected in the class $\Xi$.
This includes the following subclasses of interest. Fix $F \in \mathcal{F}$.
\begin{enumerate}[(a)]
\item {\bf Dense  projections}. Let $\mb R \in \R^n$ have independent components distributed according to $F$ and denote the distribution of $\mb R$ by $G$. This falls within our general framework above with $M$ set to $n$ and sampling without replacement. Let
%Let $\mathcal{G}_{\text{dense}}$ be the set of possible joint distributions obtained when $R \in \R^n$ has independent components distributed according to $F \in \mathcal{F}$.
%$\mathcal{G}_{\text{dense}}$ be the set of distributions with independent
%Choose any univariate distribution $F$ and
%%(for example Gaussian)
%and sample $M=n$ points without replacement from $\{1,\ldots,n\}$. Let $\mathcal{G}_{\text{dense}}$ be the set of all such joint distributions (with i.i.d.\ entries drawn from distribution $F\in \mathcal{F}$).
\[\Xi_{\text{dense}} :=\{\xi \in \Xi \mbox{ with joint distribution equal to G}\}.  \]
\item {\bf Subsampling}. Let $\mathcal{G}_{\text{subsample}}$ be the set of distributions for $R$ obtained through \eqref{eq:projR} when subsampling with replacement. Let
\[\Xi_{\text{subsample}} :=\{\xi \in \Xi: \mbox{joint distribution } G \in \mathcal{G}_{\text{subsample}} \}.  \]
\item {\bf Minimal subsampling}. Let $\Xi_{\text{minimal}}$ be the set of all parameters in $\Xi_{\text{subsample}}$ such that the close pairs threshold is $\tau=0$ and $M$ takes randomly values in the set $\{m,m+1\}$ for some positive integer $m$.
\[\Xi_{\text{minimal}} :=\{\xi\in \Xi_{\text{subsample}} \mbox{ with } \tau=0 \mbox{ and } M\in \{m,m+1\} \mbox{ for some } m \in \mathbb{N} \}.  \]
\end{enumerate}
Note that we have suppressed the dependence of the classes above on the fixed distribution $F \in \mathcal{F}$ for notational simplicity.
We define $\mathcal{F}$ to be the set of all univariate
absolutely continuous and symmetric distributions with bounded density and finite third moment.
The restriction to continuous distributions in $\mathcal{F}$ ensures that $\Xi_{\text{minimal}}$ is invariant to the choice of $F$: when $\tau\equiv 0$, every $F \in \mathcal{F}$ with $L \in \mathbb{N}$ and the distribution for $M$ fixed yields the same algorithm. Moreover the set of close pairs in $C_l$ is simply the set of pairs $(j,k)$ that have $X_{i_m j}=Z_{i_mk}$ for all $m=1,\ldots,M$, that is the set of pairs that are equal on the subsampled rows.
We note that the symmetry and boundedness of the densities in $\mathcal{F}$ and finiteness of the third moment are mainly technical conditions necessary for the theoretical developments in the following section.
We will assume without loss of generality that the second moment is equal to $1$. This condition places no additional restriction on $\Xi$ since a different second moment may be absorbed into the choice of $\tau$.

Minimal subsampling represents a very small subset of the much larger class of randomised algorithms outlined above.
However, Theorem~\ref{thm:optim} below shows that minimal subsampling is essentially always at least as good as any algorithm from the wider class, which is perhaps surprising. A beneficial consequence of this result is that we only need to search for the optimal ways of selecting $M$ and $L$; the threshold $\tau$ is fixed at $\tau=0$ and the choice of the continuous distribution $F$ is inconsequential for minimal subsampling. The choices we give in Section~\ref{sec:final_xyz} yield a subquadratic run time that approaches linear in $p$ when the interactions to be discovered are much stronger than the bulk of the remaining interactions.

\begin{figure}
\centering
\includegraphics[scale=1.55]{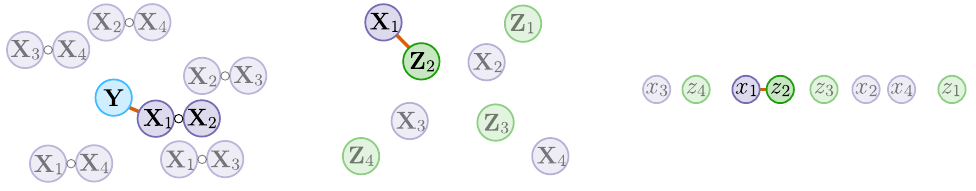}
a) \hspace{4.5cm} b) \hspace{5.1cm} c)
\caption{\label{fig:overview} Illustration of the general \emph{xyz} algorithm. The strongest interaction is the pair $(1,2)$ and $p=4$. Panel a) illustrates the interaction search among $\mb Y$ and $\matr{X}_j \circ \matr{X}_k$, panel b) shows the closest pair problem after the transformation $Z_{ij}= X_{ij} Y_i$ and panel c) depicts the closest pair problem after the data has been projected. These are the three main steps in the \emph{xyz} algorithm.}
\end{figure}

\subsection{Optimality of minimal subsampling}
\label{sec:optmin}
In this section, we compare the run time of the algorithms in $\xi \in \Xi_{\text{dense}},\Xi_{\text{subsample}}$ and $\Xi_{\text{minimal}}$ that return strong interactions with high probability.
Let $(j^*, k^*)$ be the indices of a strongest interaction pair, that is $\gamma_{j^*k^*} = \max_{j,k \in\{1,\ldots,p\}} \gamma_{jk}$. We will consider algorithms $\xi$ with $\gamma$ set to $\gamma_{j^*k^*}$. Define the power of $\xi \in \Xi$ as
\[
\mbox{Power}(\xi) :=  \pr_\xi( (j^*,k^*) \in I).
\]
%where $(j^*,k^*)$ is an interaction pair with interaction strength $\gamma_{j^* k^*}=\max_{j,k \in \{1,\ldots,p\}} \gamma_{jk}$.
For $\eta \in (0,1)$, let
\[
\Xi_{\text{dense}}(\eta) = \{\xi \in \Xi_{\text{dense}} : \mbox{Power}(\xi) \ge \eta \},
\]
and define $\Xi_{\text{subsample}}(\eta)$ and $\Xi_{\text{minimal}}(\eta)$ analogously. Note that these classes depend on the underlying $F \in \mathcal{F}$, which is considered to be fixed, and moreover that we are fixing $\gamma=\gamma_{j^* k^*}$.
%We can now constrain all parameter spaces to those that achieve at least power $\eta\in (0,1]$: Given $F \in \mathcal{F}$
%\begin{align*}
%\Xi_{\text{dense}}(\eta) &:= \{ \xi \in \Xi_{\text{dense}}:\,  \mbox{Power}(\xi) \ge \eta \} \\
%\Xi_{\text{subsample}}(\eta) &:= \{ \xi \in \Xi_{\text{subsample}}:\,  \mbox{Power}(\xi) \ge \eta\} \\
%\Xi_{\text{minimal}}(\eta) &:= \{ \xi \in \Xi_{\text{minimal}} :\,  \mbox{Power}(\xi) \ge \eta \}
%\end{align*}
%These are the sets of all algorithms in $\Xi$ with projection distribution set to $F$ and discovery probability at least $\eta$.
We consider an asymptotic regime where we have a sequence of response--predictor matrix pairs $(\mb Y^{(n)}, \mb X^{(n)}) \in \R^n \times \R^{n \times p_n}$. Write $\gamma^{(n)}_{jk}$ for the corresponding interaction strengths, and let $\gamma_1^{(n)} = \max_{j,k} \gamma^{(n)}_{jk}$. Let $f_{\mbb\gamma^{(n)}}$ be the probability mass function corresponding to drawing an element of $\mbb\gamma^{(n)}$ uniformly at random. Note that $f_{\mbb\gamma^{(n)}}$ has domain $\{0, 1/n, 2/n, \ldots, 1\}$. We make the following assumptions about the sequence of interaction strength matrices $\mbb\gamma^{(n)}$.
\begin{itemize}
\item[(A1)] There exists $c_0$ such that $|\{(j,k) : \gamma_{jk}^{(n)} = \gamma_1^{(n)}\}| \leq c_0 p_n$.
\item[(A2)] There exists $\gamma_l >0$, $\gamma_u<1$ such that $\gamma_u\geq\gamma_1^{(n)} \geq \gamma_l$ for all $n$.
\item[(A3)] There exists $\rho < 1$ such that $f_{\mbb\gamma^{(n)}}$ is non-increasing on $[\rho \gamma_1^{(n)}, \gamma_1^{(n)}) \cap \{0, 1/n, \ldots, 1\}$.
\end{itemize}
Assumption (A1) is rather weak: typically one would expect the maximal strength interaction to be essentially unique, while (A1) requires that at most of order $p_n$ interactions have maximal strength. (A2) requires the maximal interaction strength to be bounded away from 0 and 1, which is the region where complexity results for the search of interactions are of interest. As mentioned earlier, if the maximal interaction strength is 1, it will always be retained in the close-pair sets $C_l$, whilst if its strength is too close to 0, then it is near impossible to distinguish it from the remaining interactions. (A3) ensures a certain form of separation between maximal strength interactions and the bulk of the interactions.

To aid readability, in the following we suppress the dependence of quantities on $n$ in the notation.
%Given an algorithm $\xi$, the interaction threshold $\gamma$ has no bearing on the computational complexity of the procedure.
Given $\mb X$ and $\mb Y$, we may define  $T(\xi)$ as the expected number of computational operations performed by the algorithm corresponding to $\xi$.
We have the following result.
\begin{thm} \label{thm:optim}
Given $F\in \mathcal{F}$ and $\eta \in (0,1)$, there exists $n_0$ such that for all $n \geq n_0$ we have
\begin{align}
\inf_{\xi \in \Xi_{\text{minimal}}(\eta)} T(\xi) \; &=\; \inf_{\xi \in \Xi_{\text{subsample}}(\eta)} T(\xi),\label{eq:res_opt}\\
\inf_{\xi \in \Xi_{\text{minimal}}(\eta)} \frac{T(\xi)}{np^2} \; &\to \;  0, \label{eq:res_subquad}
\end{align}
and there exists $c > 0$ such that
\begin{equation} \label{eq:res_dense}
\inf_{\xi \in \Xi_{\text{dense}}(\eta)} \frac{T(\xi)}{np^2} > c.
\end{equation}
\end{thm}
The theorem shows that the optimal run time is achieved when using minimal subsampling. The last point is surprising: setting $\mb R \sim \mathcal{N}(\mb 0,\mb I)$, for example, will not improve the computational complexity over the brute-force approach and dense Gaussian projections hence do not reduce the complexity of the search. This is not caused by the larger computational effort involved in computing the dense projections: indeed even if these could be computed for free this result would remain. Rather the cost stems from the fact that dense projections have a much lower power for detecting true close pairs in the projected one-dimensional space.

\subsection{The final version of xyz} \label{sec:final_xyz}
The optimality properties of minimal subsampling presented in the previous section suggest the approach set out in Algorithm~\ref{alg:final_xyz}, which we will refer to as the \emph{xyz} algorithm.
\begin{algorithm}
\caption{\label{alg:final_xyz}Final version of the \emph{xyz} algorithm.}
\begin{algorithmic}[1]
\Statex \textbf{Input}: $\matr{X} \in \{-1,1\}^{n \times p}$, $\mb Y \in \{-1,1 \}^n$, subsample size $M$, number of projections $L$, threshold for interaction strength $\gamma$.
\Statex \textbf{Output}: $I$ set of strong interactions.
\State Form $\mb Z$ via $Z_{ij}=Y_i X_{ij}$.
\For{$l \in \{1,\ldots,L\}$ }
\State Form $\mb R \in \R^n$ as in (\ref{eq:projR}) with distribution $F=U[0,1]$ and set $\mb x=\mb X^T \mb R$, $\mb z=\mb Z^T \mb R$.
\State Find all pairs $(j, k)$ such that $x_j=z_k$ and store these in $E_l$.
\State Add to $I$ those pairs in $E_l$ for which $\gamma_{jk} \geq \gamma$.
%$|\matr{X}_j ^T \matr{Z}_k|/n > \gamma$.
\EndFor
\end{algorithmic}
\end{algorithm}
Here we are using a simplified version of the minimal subsampling proposal given in the previous section where we keep $M$ fixed rather than allowing it to be random. The reason is that the potential additional gain from allowing $M$ to be any one of two consecutive numbers with certain probabilities is minimal but necessary for Theorem \ref{thm:optim} and so the simpler approach is preferable.
We note that the uniform distribution in line 3 may be replaced with any continuous distribution to yield identical results.

To perform the equal pairs search in line 4, we sort the concatenation $(\mb x, \mb z) \in \R^{2p}$ to determine the unique elements of $\{x_1,\ldots,x_p, z_1,\ldots,z_p\}$. At each of these locations, we can check if there are components from both $x$ and $z$ lying there, and if so record their indices. This procedure, which is illustrated in Figure~\ref{fig:equal_pairs}, gives us the set of equal pairs $E$ in the form of a union of Cartesian products. The computational cost is $\mathcal{O}(p \log(p))$. This complexity is driven by the cost of sorting whilst the recording of indices is linear in $p$. We note, however, that looping through the set of equal pairs in order to output a list of close pairs of the form $(j_1, k_1), \ldots, (j_{|E|},k_{|E|})$ would incur an additional cost of the size of $E$, though in typical usage we would have $|E| =o(p)$.
\begin{figure}
\centering
\includegraphics[scale=1.5]{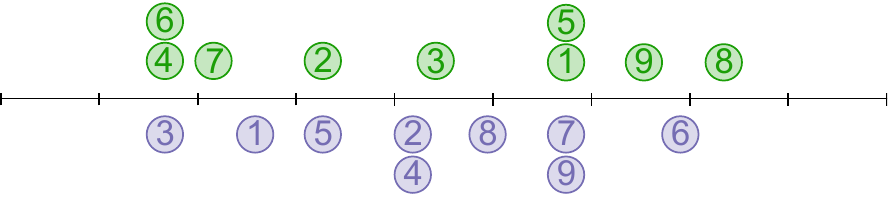}
\caption{\label{fig:equal_pairs} Illustration of an equal pairs search among components of $\mb x, \mb  z\in \R^p$ when $p=9$. The horizontal locations of blue and green circles numbered $j$ give $x_j$ and $z_j$ respectively. Sorting of $(\mb x, \mb z)$ allows traversal of the unique locations. At each of these it is checked whether points of both colours are present, and if so, the indices are recorded. Here the set of equal pairs $(\{3\} \times \{4,6\})\cup (\{5\} \times \{2\}) \cup (\{7, 9\} \times \{1,5\})$ would be returned.}
\end{figure}
Readers familiar with  locality sensitive hashing (LSH) can find a short interpretation of equal pairs search as an LSH-family in the appendix.
In the next section, we discuss in detail the impact of minimal subsampling on the complexity of the \emph{xyz} algorithm and the discovery probability it attains.

\subsection{Computational and statistical properties of xyz} \label{sec:xyz_prop}
%Lemma~\ref{lem:close_pairs} gives us
We have the following upper bound on the expected number of computational operations performed by \emph{xyz} (Algorithm~\ref{alg:final_xyz}) when the subsample size and number of repetitions are $M$ and $L$:
\begin{equation} \label{eq:complexity}
C(M,L) := \underset{\text{(i)}}{np} + L\{\underset{\text{(ii)}}{Mp} + \underset{\text{(iii)}}{p\log(p)} + \underset{\text{(iv)}}{n\E_{\xi}(|E_1|)}\}.
\end{equation}
The terms may be explained as follows: (i) construction of $\mb Z$; (ii) multiplying $M$ subsampled rows of $\mb X$ and $\mb Z$ by $\mb R \in \R^n$; (iii) finding the equal pairs; (iv) checking whether the interactions exceed the interaction strength threshold $\gamma$. Note we have omitted a constant factor from the upper bound $C(M, L)$. There is a lower bound only differing from \eqref{eq:complexity} in the equal pairs search term (iii), which is $p$ instead of $p \log(p)$. It will be shown that (iv) is the dominating term and therefore the upper and lower bound are asymptotically equivalent, implying the bounds are tight.

An interaction with strength $\gamma$ is retained in $E_1$ with probability $\gamma^M$. Hence it is present in the final set of interactions $I$ with probability
\begin{equation} \label{eq:eta}
\eta(M,L) = 1-(1-\gamma^M)^L.
\end{equation}
The following result demonstrates how the \emph{xyz} algorithm can be used to find interactions whilst incurring only a subquadratic computational cost.
\begin{thm} \label{thm:minsamp}
Let $F_{\Gamma}$ be the distribution function corresponding to a random draw from the set of interaction strengths $\{\gamma_{jk}\}_{j,k \in \{1,\ldots,p\}}$. Given an interaction strength threshold $\gamma$, let $1-F_{\Gamma}(\gamma)=c_1/p$.
Define $\gamma_0 = p^{-1/M}$ and let $c_2$ be defined by $1-F_{\Gamma}(\gamma_0)=c_2p^{\log (\gamma)/\log (\gamma_0)-1}$. We assume that $\gamma_0 < \gamma$. Finally given a discovery threshold $\eta' \in [1/2,1)$ let $L$ be the minimal $L'$ such that $\eta(M, L') \geq \eta'$. Ignoring constant factors we have
\[
C(M, L) \leq \log\{1/(1-\eta')\}(1 + c_1+c_2)[\{1 + 1/\log(\gamma_0^{-1})\}\log(p) + n]p^{1 + \log(\gamma)/\log(\gamma_0)}.
\]
\end{thm}
If $n \gg \log(p)$ and $\gamma_0$ is bounded away from 1 we see that the dominant term in the above is
\begin{equation} \label{eq:minsamp_run time}
cnp^{1 + \log(\gamma)/\log(\gamma_0)},
\end{equation}
where $c =\log\{1/(1-\eta')\}(1 + c_1+c_2)$. %Note that $\eta^{\prime}$ completely decouples from $n$ and $p$ in the total run time, which means we can have any power asymptotically without influencing the run time in terms of $n$ and $p$.
Typically we would expect $\gamma$ to be such that $|\{\gamma_{jk} : \gamma_{jk} > \gamma\}| \sim p$ as only the largest interactions would be of interest: thus we may think of $c_1$ as relatively small. If $M$ is such that $\gamma_0$ is also larger than the bulk of the interactions, we would also expect $c_2$ to be small.
Indeed, suppose that the proportion of interactions whose strengths are larger than $\gamma_0$ is $1-F_{\Gamma}(\gamma_0) = c_1' /p$. Then $c_2 = c_1' / p ^{\log(\gamma)/\log(\gamma_0)} < c_1'$.
As a concrete example, if $\gamma=0.9$ and $M$ is such that $\gamma_0=0.55$, the exponent in \eqref{eq:minsamp_run time} is around 1.17, which is significantly smaller than the exponent of 2 that a brute-force approach would incur; see also the examples in Section~\ref{sec:experi}. Note also that when $\gamma=1$, the exponent is 1 for all $\gamma_0 < 1$: if we are only interested in interactions whose strength is as large as possible, we have a run time that is linear in $p$.

It is interesting to compare our results here with the run times of approaches based on fast matrix multiplication. By computing $\mb X^T \mb Z$ we may solve the interaction search problem \eqref{eq:inter_search1}. Naive matrix multiplication would require $\mathcal{O}(np^2)$ operations, but there are faster alternatives when $n=p$. The fastest known algorithm \citep{Williams2012} gives a theoretical run time of $\mathcal{O}(np^{1.37})$ when $n=p$. For \emph{xyz} to achieve such a run time when $\gamma_0=0.55$ for example, the target interaction strength would have to be $\gamma \geq 0.81$: a somewhat moderate interaction strength. For $\gamma > 0.81$, \emph{xyz} is strictly better; we also note that fast matrix multiplication algorithms tend to be unstable or lack a known implementation and are therefore rarely used in practice. A further advantage is that the \emph{xyz} algorithm has an optimal memory usage of $\mathcal{O}(np)$.

We also note that whilst Theorem \ref{thm:minsamp} concerns the the discovery of any single interaction with strength at least $\gamma$, the run time required to discover a fixed number interactions with strength at least $\gamma$ would only differ by a multiplicative constant. If we however want a guarantee of discovering the $p$ strongest pairs the bound in Theorem \ref{thm:minsamp} would no longer hold.

To minimise the run time in \eqref{eq:minsamp_run time}, we would like $\gamma_0$ to be larger than most of the interactions in order that $c_2$ and hence $c$ be small, yet a smaller $\gamma_0$ yields a more favourable exponent. Thus a careful choice of $M$, on which $\gamma_0$ depends, is required for \emph{xyz} to enjoy good performance. In the following we show that an optimal choice of $M$ exists, and we discuss how this $M$ may be estimated based on the data.

Clearly if for some pair $(M, L)$, we find another pair $(M', L')$ with $\eta(M', L') > \eta(M, L)$ but $C(M', L') \leq C(M, L)$, we should always use $(M',L')$ rather than $(M, L)$. It turns out that there is in fact an optimal choice of $M$ such that the parameter choice is not dominated by any others in this fashion. Define
\begin{equation} \label{eq:bestM}
M^* = \argmin{M \in \N} \bigg\{-\frac{1}{\log(1-\gamma^M)}\bigg( Mp + p \log(p) + n\sum_{j,k}\gamma_{jk}^M\bigg)\bigg\},
\end{equation}
where it is implicitly assumed that the minimiser is unique. This will always be the case except for peculiar values of $\gamma$.
%\gian{here} and $\gamma_{jk}$.

\begin{prop} \label{prop:Pareto}
Let $L \in \N$. If $(M', L') \in \N^2$ has $\eta(M', L') \geq \eta(M^*, L)$, then also $C(M', L')\geq C(M^*, L)$ with the final inequality being strict if $M' \neq M^*$ and $M^*$ is a unique minimiser.
\end{prop}
Thus there is a unique Pareto optimal $M$. Although the definition of $M^*$ involves the moments of $F_\Gamma$, this can be estimated by sampling from $\{\gamma_{jk}\}$. We can then numerically optimise a plugin version of the objective to arrive at an approximately optimal $M$.

\section{Interaction search on continuous data}
\label{sec:cont}
In the previous section we demonstrated how the \emph{xyz} algorithm can be used to efficiently solve the simplest form of interaction search \eqref{eq:inter_search1} when both $\mb X$ and $\mb Y$ are binary. In this section we show how small modifications to the basic algorithm can allow it to do the same when $\mb Y$ is continuous, and also when $\mb X$ is continuous. We discuss the regression setting in Section~\ref{sec:regression}.

\subsection{Continuous $Y$ and binary $\mb X$}

We begin by considering the setting where $\mb X \in \{-1, 1\}^{n \times p}$, but where we now allow real-valued $\mb Y \in \R^n$. Without loss of generality, we will assume $\|\mb Y\|_1=1$. The approach we take is motivated by the observation that the inner product $\mb Y^T (\mb X_j \circ \mb X_k)$ can be interpreted as a weighted inner product of $\mb X_j \circ \mb X_k$ with the sign pattern of $\mb Y$, using weights $w_i=|Y_i|$.

With this in mind, we modify \emph{xyz} in the following way. We set $\mb Z$ to be $Z_{ij} = \sgn(Y_i)X_{ij}$.
Let $i_1,\ldots,i_M \in \{1,\ldots,n\}$ be i.i.d.\ such that $\pr(i_s=i)=w_i$. Forming the projection vector $\mb R$ using \eqref{eq:projR}, we then find the probability of $(j,k)$ being in the equal pairs set may be computed as follows.
\begin{align*}
\{\pr(\mb R^T \mb X_j = \mb R^T \mb Z_k)\}^{1/M} &= \pr\left(X_{i_sj}= \sgn(Y_{i_s}) X_{i_s k} \text{ for all }s=1,\ldots,M \right) \\
&= \pr( X_{i_1 j} = \sgn(Y_{i_1}) X_{i_1 k}) \qquad \text{as the $i_s$ are i.i.d.}\\
&= \sum_{i=1}^n \pr( X_{i_1 j} = \sgn(Y_{i_1}) X_{i_1 k} |i_1=i) \pr(i_1=i) \\
&= \sum_{i=1}^n |Y_i|\ind_{\{ X_{i j} = \sgn(Y_{i}) X_{i k}\}} \\
&= \sum_{i : \sgn(Y_i) = X_{ij} X_{ik}} Y_i X_{ij} X_{ik} =: \tilde{\gamma}_{jk},
\end{align*}
where $\mathbb{P}$ here is with respect to the randomness of $\mb R$ (and, equivalently, the random indices $i_1,\ldots,i_M$) with $\mb Y$ and $\mb X$ considered fixed. The calculation above shows that the run time bound of Theorem~\ref{thm:minsamp} continues to hold in the setting with continuous $\mb Y$ provided we replace the interaction strengths $\gamma_{jk}$ with their continuous analogues $\tilde{\gamma}_{jk}$.

As a simple example, consider the model
\begin{equation*}
Y_i=X_{i1} X_{i2}+\varepsilon_i,
\end{equation*}
with $\varepsilon_i \sim \mathcal{N}(0,\sigma^2)$ and $\mb X$ generated randomly having each entry drawn independently from $\{-1, 1\}$ each with probability $1/2$. Then for a non-interacting pair $j \neq 1,2 \textrm{ or } k \neq 1,2$, we have $\tilde{\gamma}_{j k} \approx 0.5$. For the pair $(1,2)$ we calculate an interaction strength of
\begin{align*}
\tilde{\gamma}_{12}&=\mathbb{P}(\sgn(Y_{i_1})=X_{i_1 1}X_{i_1 2})=\mathbb{P}(\sgn(X_{i_1 1} X_{i_1 2}+\varepsilon_i)= X_{i_1 1}X_{i_1 2})\\
&=\mathbb{P}(|\varepsilon_i| < 1)+\frac{1}{2}\mathbb{P}(|\varepsilon_i|>1)=\frac{1}{2}(1+\mathbb{P}(|\varepsilon_i|<1)).
\end{align*}
Note that here that probability is over the randomness in the noise $\varepsilon_i$.
A quick simulation gives the following table:
\begin{center}
\begin{tabular}{l*{6}{c}r}
$\sigma^2$ & $0.1$ & 0.25 & 0.5 & 1  & 2 & 5 \\
\hline
$\tilde{\gamma}_{12}$ & 0.99 & 0.98 & 0.92 & 0.84 & 0.76 & 0.67
\end{tabular}
\end{center}
Using Theorem \ref{thm:minsamp} and the above table we can estimate the computational complexity needed to discover the pair $(1,2)$ given a value of $\sigma^2$.

\subsection{Continuous $\mb Y$ and continuous $\mb X$}

The previous section demonstrated how resampling with non-uniform weights transforms a setup with continuous $\mb Y$ into one with binary response. If both $\matr{X}$ and $\mb Y$ are continuous, we continue to use the previous strategy to deal with the continuous response. For the  matrix $\mb X$ with continuous predictor values we cannot use weighted resampling as the weights would depend on the interaction pair of interest. In the following we examine the effects of transformations of $\mb X$ to a binary data matrix $\tilde{\mb X}$. To allow for randomized mappings, we define the transformations  via a function $g:\mathbb{R} \mapsto [0,1]$ as \[
\mathbb{P}(\tilde{X}_{ij}=1) = g(X_{ij}) \mbox{  and  }  1-\mathbb{P}(\tilde{X}_{ij}=-1)=1-g(X_{ij}),
\]
where the transformation is always applied independently for each entry of the predictor matrix and for each subsample.

The following gives the probability of $Y_i$ agreeing in sign with $\tilde{X}_{ij} \tilde{X}_{ik}$ when $i$ is sampled with probability proportional to $|Y_i|$.

\begin{prop}
\label{lem:transform}
Given the transform $\mathbb{P}(\tilde{X}_{ij}=1) = g(X_{ij})$ and  sampling an index $i_s$ according to $\mathbb{P}(i_s=i)=Y_i/ \|\mb Y\|_1$, then the probability of a match is
\begin{equation}
\label{eq:interactionstrength}
\mathbb{P}(\sgn(Y_{i_s}) = \tilde{X}_{i_s j}\tilde{X}_{i_s k}) = \frac{1}{2}+\frac{1}{2 \| \mb Y \|_1} \sum_{i=1}^n Y_i (1-2g(X_{ij}))(1-2g(X_{ik})).
\end{equation}
\end{prop}
Thus we may define a continuous analogue of the interaction strength $\gamma_{jk}$ based on the transform given by $g$ as
\begin{equation*}
\gamma_{jk}^g = \frac{1}{2}+\frac{1}{2 \|\mb  Y \|_1} \sum_{i=1}^n Y_i (1-2g(X_{ij}))(1-2g(X_{ik})).
\end{equation*}
These quantities may be substituted into Theorem~\ref{thm:minsamp} to yield the following upper bound on expected run time when using \emph{xyz} on transformed data.
\begin{cor}
\label{cor:generalization}
Let $F_{\Gamma^g}$ be the distribution function corresponding to a random draw from the set of interaction strengths $\{\gamma^g_{jk}\}_{j,k \in \{1,\ldots,p\}}$. Given an interaction strength threshold $\gamma$, let $1-F_{\Gamma^g}(\gamma)=c_1/p$.
Define $\gamma_0 = p^{-1/M}$ and let $c_2$ be defined by $1-F_{\Gamma}(\gamma_0)=c_2p^{\log (\gamma)/\log (\gamma_0)-1}$. We assume that $\gamma_0 < \gamma$. Finally given a discovery threshold $\eta' \in [1/2,1)$ let $L$ be the minimal $L'$ such that $\eta(M, L') \geq \eta'$. Ignoring constant factors we have
\[
C(M, L) \leq \log\{1/(1-\eta')\}(1 + c_1+c_2)[\{1 + 1/\log(\gamma_0^{-1})\}\log(p) + n]p^{1 + \log(\gamma)/\log(\gamma_0)}.
\]
\end{cor}
%\begin{proof}
%After transforming $\mb X$ to $\tilde{ \mb X}$ and subsampling according to weights $|Y_i|/\|Y\|_1$ we consider the data set $(\tilde{ \mb X},\sgn(\mb Y))$, which is a binary data set, hence Theorem \ref{thm:minsamp} applies.
%\end{proof}
The expected computational costs depends critically on the distribution of the interaction strengths $F_{\Gamma^g}$. To gain a better understanding of what impact different transformations have on this distribution and subsequently on run time we will study the following simple model for $(\mb Y, \mb X) \in \R^n \times \R^{n \times p}$:
\begin{equation} \label{eq:inter_mod}
Y_i = X_{ij^*} X_{ik^*} + \varepsilon_i, \qquad i=1,\ldots,n,
\end{equation}
where the $\varepsilon_i$ are independent and have identical sub-exponential distributions symmetric about 0 and the rows of $\mb X$ are i.i.d. We now introduce two practically useful choices of $g$ and study their properties in the context of model \eqref{eq:inter_mod}.

\subsection*{The unbiased transform}
A natural choice for the transform $g$ is one that satisfies the unbiasedness requirement:
\begin{equation} \label{eq:unbiased}
\mathbb{E}(\tilde{X}_{ij}) = X_{ij}.
\end{equation}
It turns out that this requirement uniquely defines the transform, which we refer to as the \emph{unbiased transform}.
\begin{prop} \label{prop:unbiased}
Let $X_{ij} \in [-1,1]$. If its transformed version $\tilde{X}_{ij}$ satisfies \eqref{eq:unbiased}, then $g$ takes the form
\[
\mathbb{P}(\tilde{X}_{ij}=1) = g(X_{ij})= \frac{X_{ij}+1}{2}.
\]
Furthermore the interaction strength in \eqref{eq:interactionstrength} is given by
\begin{equation*}
\mathbb{P}(\sgn(Y_{i_s})=\tilde{X}_{i_s j}\tilde{X}_{i_s k})=\gamma_{jk}^g=\frac{1}{2}+\frac{1}{2 \|\mb Y\|_1} \sum_{i=1}^n Y_i X_{ij} X_{ik}.
\end{equation*}
\end{prop}
Proposition \ref{prop:unbiased} shows that $\gamma_{jk}^g$ is a monotone function of the inner product $\sum_{i=1}^n Y_i X_{ij}X_{ik}$.

We remark that if the entries  of $\mb X$ do not lie in $[-1, 1]$, we may divide each entry in the $i$th row by $\nu_i := \max_j |X_{ij}|$, and multiply $Y_i$ by $\nu_i^2$, for each $i$. Proposition~\ref{prop:unbiased}  will then hold for the scaled versions of $\mb Y$ and $\mb X$.
%\gian{Are we sure that the transformed interaction strength is still the same? Proposition 6 would still hold? $r_i$ is taken taken by $r_s$ and $r_u$.}
In order to describe the performance of the unbiased transform when applied to data generated by the model \eqref{eq:inter_mod}, we define the following quantities:
\[
\mathbb{E}(|X_{ij^*}X_{ik^*}|) = m_1, \,\,\, \mathbb{E}(X_{ij^*}^2X_{ik^*}^2) = m_2 \, \text{ and } \mathbb{E}(|\varepsilon_i|) = m_\varepsilon.
\]
We consider an asymptotic regime where $p=p_n$ may diverge as $n$ tends to infinity, though we suppress this in the notation.
We introduce the following assumptions.
\begin{itemize}
\item[(B1)] $m_2 ( r_u-1) \leq \mathbb{E}(X_{ij^*} X_{ik*} X_{ij} X_{ik} ) \leq m_2 (1- r_u)$, for $r_u \in (0,1)$ and $\forall \, j,k \, \in \, \{1,\ldots,p \}^2$.
\item[(B2)] The noise level satisfies the bound
\[
\frac{1}{1- r_u} > 1+ \frac{m_\epsilon}{m_1}.
\]
\item[(B3)] Let $p$ be such that be such that
\[
\frac{\log(n)\log(p)}{n} \overset{n \to \infty}{\to} 0.
\]
\end{itemize}
(B1)
%is an identifiability assumption that
ensures non-interactions are not too strongly correlated to the actual interaction pair $(j^*, k^*)$.
%(B2) ensures that that the signal to noise ratio is large enough and (B3) bounds fluctuations coming from finite sample behaviour.
Note that (B3) allows for high-dimensional settings with $p \gg n$.
\begin{thm}
\label{thm:unbiased}
Assume all entries of $\mb X$ have mean zero and lie in $[-1,1]$ almost surely. Further assume (B1)--(B3) hold. When $M$ and $L$ are as in Corollary~\ref{cor:generalization} and the unbiased transform is used, we have
\[
C(M,L) = o_{\mathbb{P}} \Big (np^{1+\delta +\frac{\log(1/2+m_2/2(m_1  +m_\varepsilon))}{\log(1/2+m_2(1-r_u)/2m_1 )}} \Big)
\]
for any $\delta >0$. Here $\mathbb{P}$ is with respect to the randomness in $\mb X$ and $\mbb\varepsilon$.
\end{thm}
%The unbiased transform leads to a canonical \gian{maybe: canonical generalization of the binary interaction strength. }form of the interaction strength. It is exactly the same form that is  used by coordinate wise update rules such as in the LARS algorithm.
%\rajen{Perhaps `canonical form' is a bit vague? Also the LARS algorithm does not have coordinatewise updates right?} \gian{Yes. Im not sure why we call it LARS? It is just glmnet, right?}
Though the run time above can often improve significantly on the worst-case quadratic run time, observe that unlike in the binary case, if there is no noise and $Y_i = X_{ij^*}X_{ik^*}$, we do not necessarily have a run time close to linear in $p$. For example, when $X_{ij} \overset{iid}{\sim} \textrm{Uniform}(-1,1)$, the interaction strength of the true interaction can be shown to equal to
\begin{equation*}
\gamma_{j^*k^*}^g =\frac{1}{2}+\frac{\sum_{i=1}^n  Y_i X_{ij^*}X_{ik^*}}{2 \| \mb Y \|_1}=\frac{1}{2}+\frac{\|\mb Y\|_2^2}{2\| \mb Y \|_1} \overset{n \rightarrow \infty}{=} \frac{13}{18}.
\end{equation*}
Substituting this into the run time given by Theorem \ref{thm:minsamp}, this would result in an expected complexity of roughly $\mathcal{O}(np^{1.47})$; this is still substantially smaller than a quadratic run time, but raises the question as to whether such a loss in speed is avoidable.

Additionally, if $\mb X$ has several outlying entries, normalising the design matrix by scaling by the row-wise maximums can  shrink $\gamma_{j^*k^*}^g$ towards $1/2$. To limit the impact of this normalisation, we can first cap the entries of $\mb X$ so their absolute value is bounded by some $c >0$. Though the resulting interaction strength will not have the form given in Proposition~\ref{prop:unbiased}, it may better discriminate between interactions of interest and noise.

Capping with $c=1$ is closely related to applying the sign transform, which we study next.

%Additionally if the data lies not in $[-1,1]$, normalizing the design $\mb X$ by its max value can be quite a speed penalty, as it shrinks $\gamma_{j^*k^*}^g$ towards $1/2$. We can ease this normalization by trading of bias and variance through choosing a transform that does not normalize by the max but by a smaller value and all values bellow and above get mapped to either $-1$ or $1$. This creates a higher interaction strength at the cost of also increasing the interaction strength for non interactions. The sign transform takes this to the extreme, as it effectively chooses a normalization constant of $0$. We discuss the sign transform in the next section.

\subsection*{The sign transform}
We now consider the \emph{sign transform} given by $\tilde{X}_{ij}=\sgn(X_{ij})$; if there are zero cases we use a coin toss to map them to $\{-1,1\}$. For the sign transform we have $g(X_{ij})=2 \, \sgn(X_{ij})-1$ and so the interaction strength is given as:
\begin{equation*}
\mathbb{P}(\sgn(Y_{i_s})=\tilde{X}_{i_s j}\tilde{X}_{i_s k})=\gamma_{jk}^g=\frac{1}{2}+\frac{1}{2 \|\mb Y\|_1} \sum_{i=1}^n Y_i \sgn(X_{ij}) \sgn(X_{ik}).
\end{equation*}
The sign transform recovers the close to linear run time achieved in the binary case when a interaction is perfect as now if $Y_i = X_{ij^*}X_{ik^*}$, we have $\gamma^g_{j^*k^*}=1$. Also the sign transform is not adversely affected by the presence of outlying entries in $\mb X$, and for our theory we can relax the assumption that the entries of $\mb X$ are in $[-1,1]$ to here only requiring that they have a subexponential distribution. To facilitate comparison with the unbiased transform, we impose assumptions analogous to (B1)--(B3):
\begin{itemize}
\item[(C1)] $r_s/2 \leq \mathbb{P}(X_{ij} < 0| X_{ik} , X_{ij^*}, X_{ik^*} ) \leq 1-r_s/2$, for $r_s \in (0,1)$ and $\forall \, j,k \, \in \, \{1,...,p \}^2$.
\item[(C2)] The noise level satisfies
\[
\frac{1}{1-r_s} > 1+ \frac{m_\epsilon}{m_1}.
\]
\item[(C3)] Let $p$ be such that
\[
\frac{\log(p)^5}{n} \overset{n \to \infty}{\to} 0.
\]
\end{itemize}
%Assumptions (C1)--(C3) are analogues of (B1)--(B3) for the sign transform.
\begin{thm} \label{thm:signdist}
Suppose that each entry of $\mb X$ has a mean-zero subexponential distribution. Further assume (C1)--(C3). When $M$ and $L$ are as in Corollary~\ref{cor:generalization} and the sign transform is used, we have
\[
C(M,L)=o_{\mathbb{P}} \Big(np^{1+\delta + \frac{\log(1/2+m_1/2(m_1+m_\varepsilon))}{\log(1-r_s)}} \Big)
\]
for any $\delta > 0$.
Here $\mathbb{P}$ is with respect to the randomness in $\mb X$ and $\varepsilon$.
\end{thm}
%The sign transform yields an interaction strength of $1$ if the interaction is perfect ($m_\varepsilon=0$).
Both transforms yield a run time of the form $o_{\pr}(np^{\alpha})$. Comparing the exponents $\alpha$ we have:
\begin{itemize}
\item[$ $] unbiased transform:
\[
 \alpha_u = 1+\frac{\log(1/2+m_2/2(m_1  +m_\varepsilon))}{\log(1/2+m_2(1-r_u)/2m_1 )}
\]
\item[$ $] sign transform:
\[
\alpha_s= 1+\frac{\log(1/2+m_1/2(m_1+m_\varepsilon))}{\log(1/2+(1-r_s)/2)}.
\]
\end{itemize}
For bounded data $\mb X \in [-1,1]^{n \times p}$ and when
%the noise variance is small compared to the first product moment (
$m_\varepsilon \ll m_1$,  we have $m_1/2(m_1+m_\varepsilon) \approx 1/2$ so that $\alpha_s = 1$ whereas $\alpha_u > 1$. Hence in case of a strong signal the sign transform can give a smaller run time than the unbiased transform.

\section{Application to Lasso regression}
\label{sec:regression}
Thus far we have only considered the simple version of the interaction search problem \eqref{eq:inter_search1} involving finding pairs of variables whose interaction has a large dot product with $\mb Y$. In this section we show how any solution to this, and in particular the \emph{xyz} algorithm, may be used to fit the Lasso \citep{tibshirani96regression} to all main effects and pairwise interactions in an efficient fashion.

Given a response $\mb Y \in \R^n$ and a matrix of predictors $\mb X \in \R^{n \times p}$, let $\mb W \in \R^{n \times p(p+1)/2}$ be the matrix of interactions defined by
\[
\mb W = (\matr{X}_1 \circ \matr{X}_1, \matr{X}_1 \circ \matr{X}_2, \cdots, \matr{X}_1 \circ \matr{X}_p, \matr{X}_2 \circ \matr{X}_2, \matr{X}_2 \circ \matr{X}_3, \cdots, \matr{X}_p \circ \matr{X}_p).
\]
We will assume that $\mb Y$ and the columns of $\mb X$ have been centred. Note that the centring of $\mb X$ means the $\mb W$ implicitly contains main effects terms. Let $\tilde{\mb W}$ be a version of $\mb W$ with centred columns.
Consider the Lasso objective function
\begin{equation} \label{eq:Lasso_obj}
(\hat{\bm{\beta}},\hat{\bm{\theta}})=\underset{\bm{\beta} \in \mathbb{R}^p, \bm{\theta} \in \mathbb{R}^{p(p+1)/2}}{\operatorname{argmin}}  \bigg\{\frac{1}{2n}\|\mb Y-\matr{X}\bm{\beta} -\tilde{\mb W}\bm{\theta}\|_2^2 +\lambda(\|\bm{\beta}\|_1+\|\bm{\theta}\|_1)\bigg\}.
\end{equation}
Note that since the entire design matrix in the above is column-centred, any intercept term would always be zero.

In order to avoid a cost of $\mathcal{O}(np^2)$ it is necessary to avoid explicitly computing $\mb W$.
To describe our approach, we first review in Algorithm~\ref{alg:active} the active set strategy employed by several of the fastest Lasso solvers such as \texttt{glmnet} \citep{friedman2010regularization}. We use the notation that for a matrix $\mb M$ and a set of column indices $H$, $\mb M_H$ is the submatrix of $\mb M$ formed from those columns indexed by $H$. Similarly for a vector $\mb v$ and component indices $H$, $\mb v_H$ is the subvector of $\mb v$ formed from the components of $\mb v$ indexed by $H$.

\begin{algorithm}
\caption{\label{alg:active}Active set strategy for Lasso computation}
\begin{algorithmic}[1]
\Statex \textbf{Input}: $\matr{X}$, $\mb Y$ and grid of $\lambda$ values $\lambda_1 > \cdots > \lambda_L$.
\Statex \textbf{Output}: Lasso solutions $\hat{\mbb \beta}_{\lambda_l}$ and $\hat{\mbb \theta}_{\lambda_l}$ at each $\mb \lambda$ on the grid.
\For{$l \in \{1,\ldots,L\}$ }
\State If $l=1$ set $A, B=\emptyset$; otherwise set $A=\{k:\hat{\beta}_{\lambda_{l-1},k}\neq 0\}$ and $B=\{k:\hat{\theta}_{\lambda_{l-1},k}\neq 0\}$.
\State Compute the Lasso solution $(\hat{\mbb\beta}, \hat{\mbb\theta})$ when $\lambda=\lambda_l$ under the additional constraint that $\hat{\mbb\beta}_{A^c}=0$ and $\hat{\mbb\theta}_{B^c}=0$.
\State Let $U=\{k:|\mb X_k^T(\mb Y-\mb X_A\hat{\mbb\beta}_A - \tilde{\mb W}_B\hat{\mbb \theta}_B)|/n >\lambda_l\}$ and $V=\{k:|\tilde{\mb W}_k^T(\mb Y-\mb X_A\hat{\mbb\beta}_A -\tilde{\mb W}_B\hat{\mbb \theta}_B)|/n >\lambda_l\}$ be the set of coordinates that violate the KKT conditions when $(\hat{\mbb \beta}, \hat{\mbb\theta})$ is taken as a candidate solution.
\State If $U$ and $V$ are empty, we set $\hat{\mbb\beta}_{\lambda_l} = \hat{\mbb\beta}$, $\hat{\mbb\theta}_{\lambda_l}=\hat{\mbb\theta}$. Else we update $A = A \cup U$ and $B=B \cup V$ and return to line 3.
\EndFor
\end{algorithmic}
\end{algorithm}

As the sets $A$ and $B$ would be small, computation of the Lasso solution in line 3 is not too expensive. Instead line 4, which performs a check of the Karush--Kuhn--Tucker (KKT) conditions involving dot products of all interaction terms and the residuals, is the computational bottleneck: a naive approach would incur a cost of $\mathcal{O}(np^2)$ at this stage.

There is however a clear similarity between the KKT conditions check for the interactions and the simple interaction search problem \eqref{eq:inter_search1}. Indeed the computation of $V$, the set containing all interactions that violate the KKT conditions, may be expressed in the following way:
\begin{gather} \label{eq:inter_search2}
\text{Keep all pairs } (j,k) \text{ for which } |(\mb Y - \mb X_A \hat{\mbb \beta}_A - \tilde{\mb W}_B\hat{\mbb \theta}_B)^T(\matr{X}_j \circ \matr{X}_k)/n| > \lambda_l.
\end{gather}
Note that since $\mb Y - \mb X_A \hat{\bm{\beta}}_A - \tilde{\mb W}_B\hat{\bm{\theta}}_B$ is necessarily centered, there is no need to center the interactions in \eqref{eq:inter_search2}.
In order to solve \eqref{eq:inter_search2} we can use the \emph{xyz} algorithm, setting $\gamma$ in Algorithm~\ref{alg:final_xyz} to $\lambda_l$ and $\mb Y$ to each of $\pm(\mb Y -\mb X_A \hat{\bm{\beta}}_A- \tilde{\mb W}_B\hat{\mbb \theta}_B)$ in turn.

Precisely the same strategy of performing KKT condition checks using \emph{xyz} can be used to accelerate computation for interaction modeling for a variety of variants of the Lasso such as the elastic net \citep{zou05regularization} and $\ell_1$-penalised generalised linear models. Note also that it is straightforward to use a different scaling for the penalty on the interaction coefficients in \eqref{eq:Lasso_obj}, which may be helpful in practice.

\section{Experiments}
\label{sec:experi}
To test the algorithm and theory developed in the previous sections, we run a sequence of experiments on real and simulated data.

\subsection{Comparison of minimal subsampling and dense projections}
\label{sec:gaussexp}
One of the surprising outcomes of our theoretical analysis is extent of the suboptimality of Gaussian random projections, which whilst they suffice for the conclusion of the Johnson--Lindenstrauss Lemma, are not well-suited for our purposes here (see Theorem~\ref{thm:optim}). We can explicitly compute the probability of retaining an interaction of strength $\gamma$ in $E_1$ for both dense Gaussian projections $\xi_{Gauss}$ and minimal subsampling $\xi_{minimal}$ given an equal computational budget. We consider various values of $p$ ranging from $10$ up to $10^6$ and we fix $n=1000$. We set $L=1$ and select other parameters of the algorithms to ensure the average size of $E_1$ is equal to $p$ in the setting when all interaction strengths are equal to 0.5. Specifically we make the following choices.
\begin{itemize}
\item $\xi_{Gauss}$: the close pairs threshold $\tau \geq 0$ is the $1/p$--quantile of the distribution of $|W|$ when~$W\sim N(0,0.5n)$.
\item $\xi_{minimal}$: the subsample size $M=\lceil\log(1/p)/\log(0.5)\rceil$.
\end{itemize}
%This way we ensure that the same number of pairs is considered after each projection. The size of $L$ is set to $1$ as it does not matter for the discovery probability how many projections are conducted, as long as we conduct the same number we can compare the two approaches.
We then plot the probability $\eta$ of discovering an interaction of strength $\gamma$, as a function of $\gamma$ for different values of $p$ (Figure \ref{fig:asympruntt}). For $\xi_{minimal}$, $\eta$ is given in equation~(\ref{eq:eta}). For $\xi_{Gauss}$, $\eta$ is the $1/p$--quantile of the distribution of $|W|$ when $W \sim N(0,n(1-\gamma))$.
%We simulate $\eta$ for values of $p$ ranging from $10$ up to $10^6$. We can clearly see that in terms of discovery probability $\xi_{minimal}$ outperforms $\xi_{Gauss}$ by magnitudes.
%The impact of this different behaviour on the run time is also highlighted by Theorem~\ref{thm:optim}.

\subsection{Scaling}

In this experiment we test how the \emph{xyz} algorithm scales on a simple test example as we increase the dimension $p$. We generate data $\matr{X} \in \mathbb{R}^{n \times p}$ with each entry sampled independently uniformly from $\{-1,1\}$. We do this for different values of $p$, ranging from $1000$ to $30\,000$: this way for the largest $p$ considered there are more than $400$ million possible interactions. Then for each $\mb X$ we construct  response vectors $\mb Y$ such that only the pair $(1,2)$ is a strong interaction with an interaction strength taking values in $\{0.7,0.8,0.9\}$. Through this construction, if $n$ is large enough, all the pairs except $(1,2)$ will have an interaction strength around $0.5$, and very few will have one above $0.55$. We thus set $M$ so that $\gamma_0=p^{-1/M} \approx 0.55$. Since the only strong interaction is $(1,2)$, we set $\gamma=\gamma_{12}$
Each data set configuration determined by $p$ and $\gamma_{12}$ is simulated $300$ times and we measure the time it takes \emph{xyz} to find the pair $(1,2)$. In Figure~\ref{fig:asympruntt} we plot the average run time against the dimension $p$ with the different choices for $\gamma_{12}$ highlighted in different colours.

Theorem~\ref{thm:minsamp} indicates that the run time should be of the order $np^{1+\log(\gamma)/\log(\gamma_0)}$. We see that the experimental results here are in close agreement with this prediction.

\begin{figure}[!ht]
\begin{center}
\hspace{-1.3cm}
\includegraphics[scale=0.9]{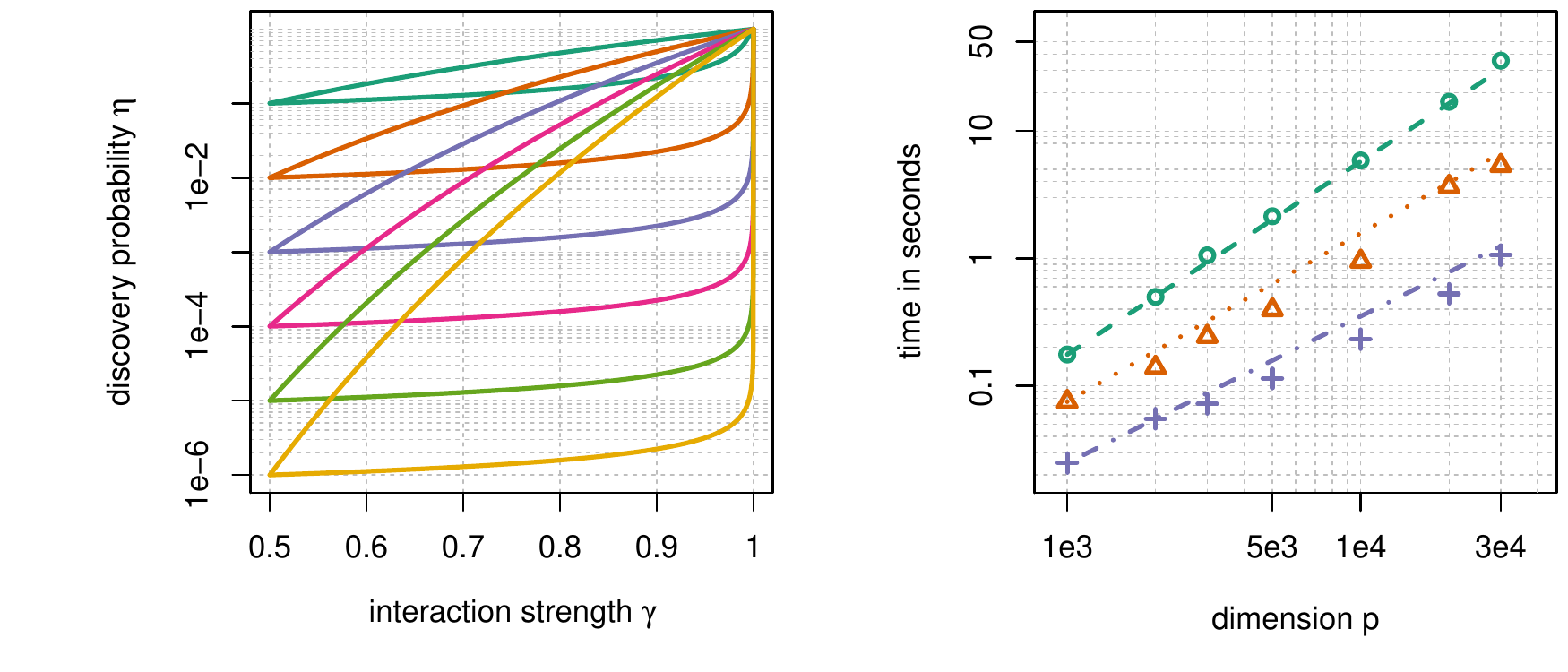}
\end{center}
\vspace{-0.7cm}
\caption{\label{fig:asympruntt} Left panel: Discovery probability as a function of $\gamma$ for different values of $p \in \{10^{1},\ldots,10^{6}\}$ (colours decreasing in $p$ from yellow $p=10^6$ to green $p=10$). The lower lines correspond to the dense Gaussian projections, the upper lines to minimal subsampling. It can be seen that the discovery probability for minimal subsampling is much higher (up to factor $10^4$) than for Gaussian projections. Right panel: Time to discover the interaction pair as a function of the data set dimension $p$. Lines correspond to the theoretical prediction (with the intercept chosen based on the data points) and symbols give the actual measured run time. Colour coding: green $\gamma=0.7$, orange $\gamma=0.8$ and purple $\gamma=0.9$.}
\end{figure}

\subsection{Run on SNP data}
\label{absolrun time}

In the next experiment we compare the performance of \emph{xyz} to its closest competitors on a real data set. For each method we measure the time it takes to discover strong interactions. We consider the LURIC data set \citep{Winkelmann2001}, which contains data of patients that were hospitalised for coronary angiography. We use a preprocessed version of the data set that is made up of $n=859$ observations and $687\,253$ predictors. The data set is binary. The response $\mb Y$ indicates coronary disease ($1$ corresponding to affected and $-1$ healthy) and $\matr{X}$ contains Single Nucleotide Polymorphisms (SNPs) which are variations of base-pairs on DNA. The response vector $\mb Y$ is strongly unbalanced: there are $681$ affected cases ($Y_i=1$) and $178$ unaffected ($Y_i=-1$).

To get a contrast of the performance of \emph{xyz} we compare it to \emph{epiq} \citep{Arkin2014}, another method for fast high-dimensional interaction search. In order for \emph{epiq} to detect interactions it needs to assume the model
\begin{equation}
\label{epiqmodel}
Y_i=\alpha_{j^*k^*} X_{ij^*}X_{ik^*} +\varepsilon_i,
\end{equation}
where $\varepsilon_i \sim \mathcal{N}(0,\sigma^2)$. It then searches for interactions by considering the test statistics
\begin{equation*}
T_{jk}=(\mb R^T (\mb Y \circ \matr{X}_j)) (\mb R^T \matr{X}_k)
\end{equation*}
where $\mb R \sim \mathcal{N}(\mb 0,\mb I)$. These are used to try to find the pair $(j^*,k^*)$, which is assumed to be the pair for which the inner product $\mb Y^T(\matr{X}_j\circ\matr{X}_k)$ is maximal. It is an easy calculation to show that $\mathbb{E}(T_{jk})=\mb Y^T(\matr{X}_j\circ\matr{X}_k)$. To maximise the inner product on the right, \emph{epiq} considers pairs where $T_{jk}^2$ is large by looking at pairs where both $(\mb R^T (\mb Y \circ \matr{X}_j))^2$ and $(\matr{R}^T \matr{X}_k)^2$ are large. While the approach of \emph{epiq} is somewhat related to \emph{xyz}, there are no bounds available for the time it takes to find strong interactions.

We also compare both methods to a naive approach where we subsample a fixed number of interactions uniformly at random, and retain the strongest one. We refer to this as \emph{naive search}.

At fixed time intervals we check for the strongest interaction found so far with all three methods. We plot the interaction strength as a function of the computational time (Figure \ref{absrun time}). All three methods eventually discover interactions of very similar strength and it would be a hasty judgement to say whether one significantly outperforms the others. \emph{xyz} nevertheless discovers the strongest interactions on average for a fixed run time compared to the other two approaches. To get a clearer picture we run two additional experiments on a slight modification of the LURIC data set. We implant artificial interactions where we set the strength to $\gamma_{12}=0.8$ and another example with $\gamma_{12}=0.9$. In these two experiments \emph{xyz} clearly outperforms all other methods considered (Figure \ref{absrun time}; panels 3 and 4). Besides \emph{xyz} being the fastest at interaction search, it also offers a probabilistic guarantee that there are no strong interactions left in the data. This guarantee comes out of Theorem \ref{thm:minsamp}. To run \emph{xyz} we have to calculate the optimal subsample size (\ref{eq:bestM}) for use of minimal subsampling:
\begin{equation*}
M^*= \argmin{M \in \N} \bigg\{-\frac{1}{\log(1-\gamma^M)}\bigg( Mp + p \log(p) + n\sum_{j,k}\gamma_{jk}^M\bigg)\bigg\} =21.
\end{equation*}
The sum in this optimisation can be approximated by uniformly sampling over pairs.
%and adjusting the normalization of the sum.
Assume we have an interaction pair $(j^*,k^*)$ with interaction strength $\gamma_{j^*k^*}=0.85$ and say the rest of the pairs $(j,k)$ have an interaction strength of no more than $\gamma_{jk} \leq 0.55$. The probability that we discover this pair in one run ($L=1$) of the \emph{xyz} algorithm is $\gamma_{j^*k^*}^{21}$. Therefore the probability of missing this pair after $L=100$ runs is given by
\[(1-\gamma_{j^*k^*}^{21})^L \approx 0.03. \]
Note that the number of possible interactions is $p(p-1)/2 \approx 10^{11}$. The whole search took $280$ seconds. Naive search offers a similar guarantee, however it is extremely weak. The probability of not discovering the pair after drawing $pL$ samples (with $L=100$) is bounded by $[1-2/\{p(p-1)\}]^{Lp} \approx 0.999$.
\begin{figure}[!ht]
\begin{center}
\hspace{0.2cm}
\includegraphics[scale=1]{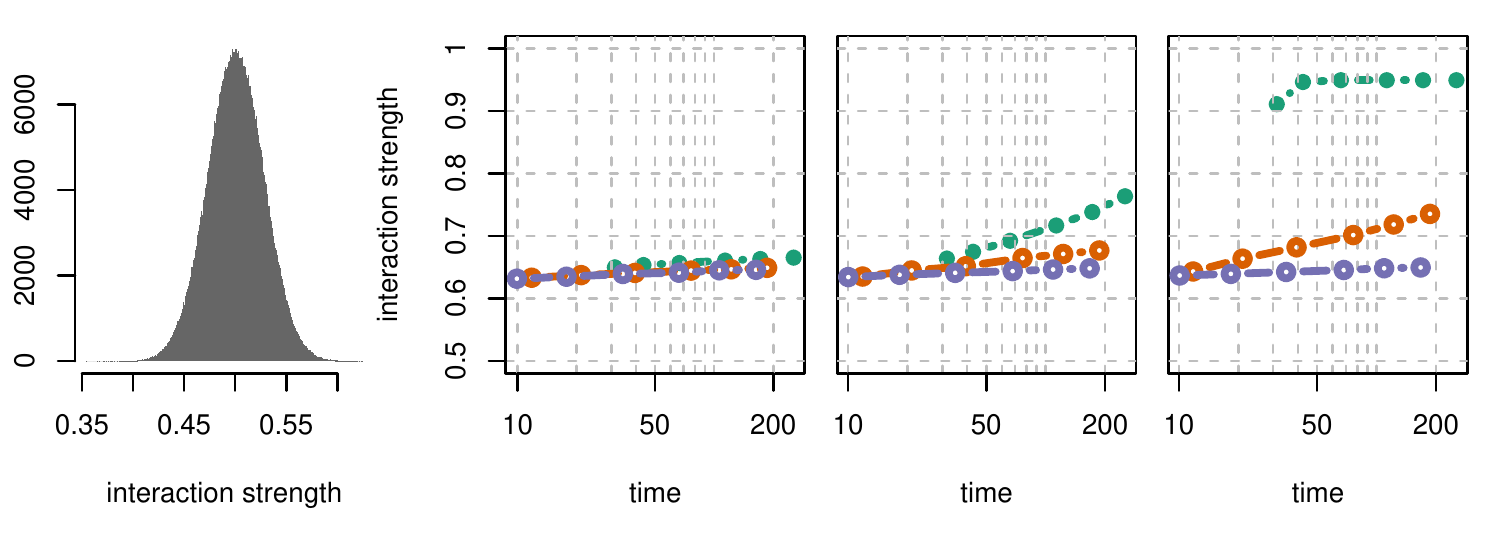}
\end{center}
\vspace{-0.6cm}
\caption{\label{absrun time} Left: Histogram of interaction strength of $10^6$ interaction pairs, sampled at random from the more than  $10^{11}$ existing pairs from the LURIC data set. The right three panels show the interaction strength of the discovered pairs as a function of the computation time for \emph{xyz} (green), \emph{epiq} (orange) and naive search (purple). The first panel gives results on the the original LURIC data set, and the second and third (rightmost) panels show results with an implanted interaction with strengths $\gamma_{12}=0.8$ and $\gamma_{12}=0.95$ respectively. It can be clearly seen that \emph{xyz} outperforms its competitors by a large margin.}
\end{figure}
If we consider the run time guarantee from Theorem \ref{thm:minsamp}, the dominating term in the complexity of \emph{xyz} in terms of $p$ is
\begin{equation*}
p^{1+\frac{\log(0.85)}{\log(0.55)}} \approx p^{1.27}.
\end{equation*}
This may be compared to the expected run time of order $p^2$ for naive search, which means that \emph{xyz} is about $30\,000$ times faster than naive search (when $p=687\,253$). In the empirical comparison this factor is around $20\,000$.

\subsection{Regression on artificial data}

In this section we demonstrate the capabilities of \emph{xyz} in interaction search for continuous data as explained in Section~\ref{sec:cont}. We simulate two different models of the form (\ref{eq:inter_mod}):
\begin{equation*}
Y_i = \mu + \sum_{j=1}^p X_{ij} \beta_j + \sum_{k=1}^p \sum_{j=1}^{k-1} X_{ij} X_{ik} \theta_{jk}  +\varepsilon_i.
\end{equation*}
We consider three settings. For all three settings we have $n=1000$. We let $p \in \{250,500,750,\\ 1000\}$. Each row of $\matr{X}$ is generated i.i.d. as $\mathcal{N}(\mb 0,\matr{\Sigma})$. The magnitudes of both the main and interaction effects are chosen uniformly from the interval $[2,6]$ ($20$ main effects and $10$ interaction effects) and we set $\varepsilon_i \sim \mathcal{N}(0,1)$. The three settings we consider are as follows.
\begin{enumerate}
\item $\matr{\Sigma} = \matr{I} \in \mathbb{R}^{p \times p}$, we generate a hierarchical model: $\theta_{jk} \neq 0$ $\Rightarrow$ $\beta_j \neq 0$ and $\beta_k \neq 0$. We first sample the main effects and then pick interaction effects uniformly from the pairs of main effects.
\item  $\matr{\Sigma} = \matr{I} \in \mathbb{R}^{p \times p}$, we generate a strictly non-hierarchical model: $\theta_{jk} \neq 0$ $\Rightarrow$ $\beta_j = 0$ and $\beta_k = 0$. We first sample the main effects and then pick interaction effects uniformly from all pairs excluding main effects as coordinates.
\item We repeat the setting $2$ with a data set that contains strong correlations. We create a dependence structure in $\matr{X}$, by first generating a DAG with on average $10$ edges per node. Each node is sampled so that it is a linear function of its parents plus some independent centred Gaussian noise, with a variance of $10\%$ the variance coming from the direct parents. The resulting correlation matrix then unveils for each variable $\matr{X}_j$ a substantial number of variables strongly correlated to $\matr{X}_j$ (There is usually around $10$ variables with a correlation of above $0.9$). Such a correlation structure will make it easier to detect pairs of variables whose product can serve as strong predictor of $\mb Y$, even though it has not been included in the construction of $\mb Y$.
\end{enumerate}
We run three different procedures to estimate the main and interaction effects.
\begin{itemize}
\item \textbf{Two-stage Lasso:}
We fit the Lasso to the data, and then run the Lasso once more on an augmented design matrix containing interactions between all selected main effects.
Complexity analysis of the Least Angle Regression (LARS) algorithm \citep{efron04least} suggests the computational cost would be $\mathcal{O}(np \min(n,p))$, making the procedure very efficient. However, as the results show, it struggles in situations such as that given by model 2, where a main effects regression will fail to select variables involved in strong interactions.
\item \textbf{Lasso with all interactions: } Building the full interaction matrix and computing the standard Lasso on this augmented data matrix.
%This approach requires both a computational and memory complexity of $\mathcal{O}(np^2)$.
Analysis of the LARS algorithm would suggest the computational complexity would be in the order $\mathcal{O}(np^2\min(n, p^2))$.
Nevertheless, for small $p$, this approach is feasible.
%If there are interactions in the data, the brute force approach will detect them at a quadratic cost.
\item \textbf{xyz: } This is Algorithm~\ref{alg:active}; we set the parameter $L$ to be $\sqrt{p}$ in order to target the strong interactions.
%The user has to set the number of projections $L$, conducted in the \emph{xyz} algorithm; we pic. The run time is roughly $\min(n,p)C(M,L)$, \rajen{I cannot remember where $\min(n,p)C(M,L)$ comes from?} \gian{I think the number of repetitions of the coordinate wise updates is $\min(n,p)$ and in each such update we can run the xyz algo. This is an uppperbound.} \rajen{I think it may be safer to remove this discussion on run time as? My understanding is that the complexity of coordinate descent is not known?} since the run time to fit a sparse linear model is typically of the order $np \min(n,p)$ and each interaction search incurs a cost of $C(M,L)$. One is usually only interested in the strong interaction effects, therefore a small $L$ will suffice (we choose $L=\sqrt{p}$). Additionally probabilistic guarantees such as in section \ref{absolrun time} assure the user that all strong interactions have been found with high probability.
\end{itemize}
The experiment (seen in Figure~\ref{fig:continuous}) shows that \emph{xyz} enjoys the favourable properties of both its competitors: it is as fast as the two-stage Lasso that gives an almost linear run time in $p$, and it is about as accurate as the estimator calculated from screening all pairs (brute-force).
\begin{figure}[!ht]

\begin{center}
\hspace{2.5cm}
\includegraphics[scale=0.85]{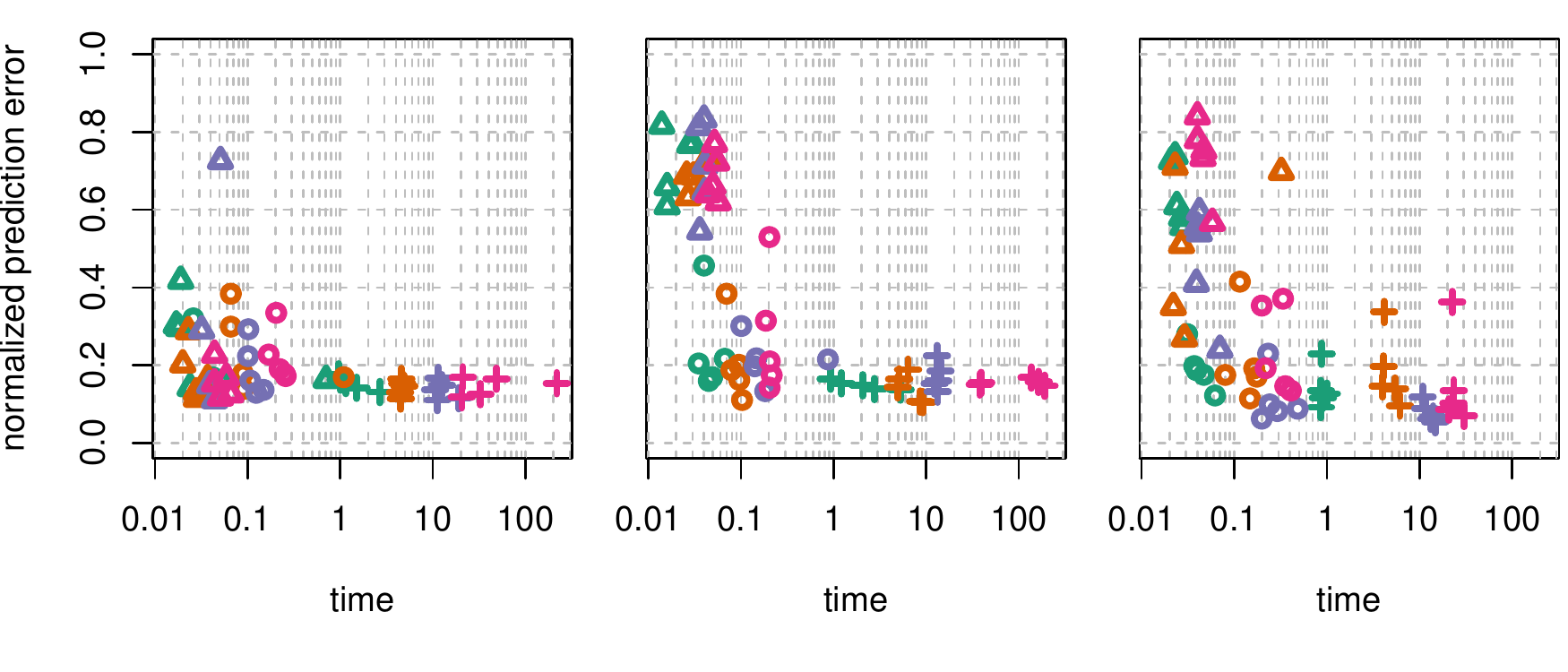}
\end{center}
\vspace{-0.9cm}
\caption{\label{fig:continuous} Normalised $\ell_2^2$ prediction error as a function of time in seconds. Triangle: Two-stage Lasso. Circle: \emph{xyz}-regression. Cross: Brute-force. The different colours correspond to different values of $p$: green $p=250$, orange $p=500$, purple $p=750$ and pink $p=1000$. The left panel shows the results on setting~$1$, center panel shows setting~$2$ and right panel setting~$3$.}
\end{figure}

\subsection{Regression on real data}
Here we run \textit{xyz} regression on continuous real data sets where the ground truth is unknown. On each data set we pick at random $p=2000$ variables and run \textit{xyz} and the Lasso implemented in \texttt{glmnet} with all interactions included. We subsample an increasing number of variables to vary the difficulty of the regression problem. For each sample we measure the run time and the normalized out of sample squared $\ell_2^2$ error: \[
\frac{\|\mb Y_{\text{test}} -\mb X_{\text{test}}\hat{\bm{\beta}}- \tilde{\mb W}_\text{test}\hat{\mbb \theta} \|_2^2}{\| \mb Y_{\text{test}} \|_2^2}.
\]
Experiments are run on the following three different data sets:
\begin{itemize}
\item \textbf{Riboflavin: } The Riboflavin production data set \citep{buhlmann14ribo} contains $n=71$ samples and $p=4088$ predictors (gene-expressions). The response $\mb Y$ and the design $\mb X$ are both continuous.
\item \textbf{Kemmeren: } The Kemmeren \citep{kemmeren2014} data set records knockouts of $p=6170$ genes. The data $\mb X$ is continuous. We sample $\mb Y$ randomly from the genes not present in the subsample taken from $\mb X$.
\item \textbf{Climate: } The climate data set from the CNRM model from the CMIP5 model ensemble \citep{knutti2013climate} simulates the temperature of points on the northern hemisphere which is recorded in $\mb X$. The response $\mb Y$ simulates the temperature on a random position on the southern hemisphere. The data contains $n=231$ observations.
\end{itemize}
\begin{figure}[!ht]
\begin{center}
\hspace{2.5cm}
\includegraphics[scale=0.85]{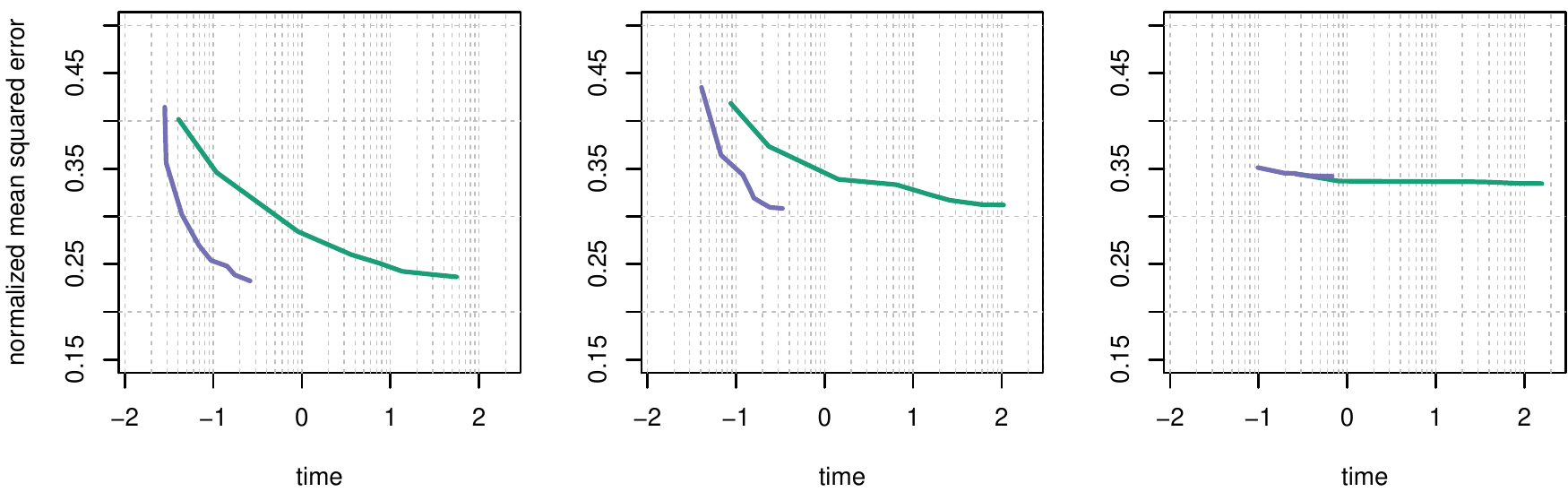}
\end{center}
\vspace{-0.4cm}
\caption{\label{fig:contcont} From left to right column the experiments correspond to Riboflavin, Kemmeren and Climate. The y-axis depicts the normalized squared error and the x-axis records the run time in seconds on the $\log_{10}$ scale. It can be seen that \emph{xyz} (purple) offers clear computational advantages while giving similar level of prediction error to the Lasso fitted to all interactions as implemented in \texttt{glmnet} (green).}
\end{figure}
For each experiment we fix the number of runs $L$ to $\sqrt{p}$ so the run time of \textit{xyz} is $\mathcal{O}(np^{1.5})$. The experiments show that the \textit{xyz} algorithm has a similar prediction performance to the Lasso applied to all interactions as implemented in \texttt{glmnet}. However \textit{xyz} is around $100$ times faster for $p=2000$. The results of all 6 experiments can be seen in Figure~\ref{fig:contcont}.

\section{Discussion} \label{sec:discuss}
In this work we exploited a relationship between closest pairs of point problems and interaction search. By solving the former problem using random projections to project points down to a one-dimensional space and then sorting the resulting projected points, we were able to produce an algorithm for interaction search that enjoys a run time that is sub quadratic under mild assumptions and when used to search for very strong interactions can be almost linear.
Though we have looked at interaction search in this paper, the basic engine for computing the large inner products between collections of vectors may have other interesting applications, for example in large-scale clustering problems. We hope to study such applications in future work.

%In this work we introduced the \emph{xyz} algorithm for efficient  search of product interactions. It is based on the translation of interaction search to a closest pairs search problem. and the relation of the latter to sorting in one dimension. The connection between high-dimensional closest pair search and its efficient one-dimensional equivalent is established through random projections.
%
%To demonstrate the efficiency of this approach we gave a full run time analysis in the case of binary data and showed that it enjoys subquadratic run time. Our results hand the user an optimal way of choosing the parameters of the \emph{xyz} algorithm to minimise its run time.
%
%We extended the \emph{xyz} algorithm to continuous data and show how it can be built into the Lasso to give an efficient algorithm that can fit high-dimensional models with both main and interaction effects.
%
%We demonstrate the correctness of our results in an experimental section and use \emph{xyz} on both real and simulated data to uncover interactions in subquadratic time.

\newpage

\appendix

\section*{Table of frequently used notation}

\begin{table}[htbp]%\caption{Table of frequently used notation}
\begin{center}% used the environment to augment the vertical space
% between the caption and the table
\begin{tabular}{r c p{10cm} }
%\toprule
\hline \\
$n, p$ & $ $ & number of observations and number of variables \\
$\mb X, \mb Y$ & & predictor matrix and response vector\\
$\mb X_j$ & & $j$th variable / column of $\mb X$\\
$\mbb\beta, \mbb\theta$ & & coefficients of main effects and interaction effects\\
$\gamma_{jk}$ & & interaction strength of the pair $(j,k)$\\
$G$ & & distribution of projection\\
$M$ & & subsample size\\
$\mb R$ & & projection vector\\
$L$ & & number of projections\\%, repetitions of the \emph{xyz} algorithm\\
$\tau, \gamma$ & & close pairs threshold and interaction strength threshold\\
$\Xi$ & & set of all configurations of the \emph{xyz} algorithm, the elements of this set are denoted by $\xi$\\
$\eta$ & & probability that a given interaction is present in the output of the \emph{xyz} algorithm\\
$\tilde{\mb X}$ & &binarized version of $\mb X$ \\
$\mb W$ & & predictor matrix containing all possible interaction pairs\\
\hline
\end{tabular}
\end{center}
\label{tab:TableOfNotationForMyResearch}
\end{table}

\section*{Appendix A}

Here we include proofs that were omitted earlier.
\subsection*{Proof of Theorem~\ref{thm:optim}}
In the following, we fix the following notation for convenience:
\begin{gather*}
\Psi = \Xi_{\text{minimal}}, \qquad \Psi(\eta) = \Xi_{\text{minimal}}(\eta),\\
\Xi = \Xi_{\text{subsample}}, \qquad \Xi(\eta) = \Xi_{\text{subsample}}(\eta).
\end{gather*}
Note that both $\Psi(\eta)$ and $\Xi(\eta)$ depend on $F$ though this is suppressed in the notation.
Also define $\Xi_{\text{all}} = \Xi \cup \Xi_{\text{dense}}$ and $\Xi_{\text{all}}(\eta) = \Xi(\eta) \cup \Xi_{\text{dense}}(\eta)$.
 We will reference the parameters levels contained in $\xi \in \Xi_{\text{all}}$ as $\xi_L$ and $\xi_\tau$. If $\xi \in \Xi$ then we will write $\xi_M$ for the distribution of the subsample size $M$.

If we let $V$ denote the complexity of the search for $\tau$-close pairs, similarly to \eqref{eq:complexity} we have that
\begin{equation} \label{eq:complexity_gen}
T(\xi) = c_1 np + L(c_2 \E_\xi M p + \E_\xi V + c_3 n\E_\xi |E_1|),
\end{equation}
where $c_1, c_2, c_3$ are constants. Suppose $\psi \in \Psi$ and $\xi \in \Xi$ have $\E_\xi |E_1| = \E_\psi |E_1|$. Then since searching for $\tau$-close pairs is at least as computationally difficult as finding equal pairs we know that $\E_\xi V \geq \E_\psi V$.

Similarly for $\xi \in \Xi_{\text{dense}}$ we have
\begin{equation} \label{eq:complexity_dense}
T(\xi) = c_1 np + L(c_2 np + \E_\xi V + c_3 n\E_\xi |E_1|).
\end{equation}

For $\xi \in \Xi_{\text{all}}$, define
\begin{align*}
\alpha(\xi) = \E_\xi |E_1| / p^2, \qquad \beta(\xi) = \pr_\xi((j^*,k^*) \in I_1)
\end{align*}
where $I_1$ is the set of candidate interactions $I$ when $L=1$.
Note that
\[
\pr_\xi((j^*,k^*) \in I) = 1-\{1-\beta(\xi)\}^{\xi_L}.
\]
Thus any $\xi \in \Xi_{\text{all}}(\eta)$ with $T(\xi)$ minimal must have $\xi_L$ as the smallest $L$ such that $1-\{1-\beta(\xi)\}^{\xi_L} \geq \eta$, whence
\begin{equation} \label{eq:L_choice}
\xi_L = \ceil{\log(1-\eta)/\log\{1-\beta(\xi)\}}.
\end{equation}
Note that $\beta(\xi)$ does not depend on $\xi_L$, so the above equation completely determines the optimal choice of $L$ once other parameters have been fixed. We will therefore henceforth assume that $L$ has been chosen this way so that the discovery probability of all the algorithms is at least $\eta$.

The proofs of \eqref{eq:res_subquad} and \eqref{eq:res_dense} are contained in Lemmas~\ref{lem:subquad} and~\ref{lem:dense} respectively.
The proof of \eqref{eq:res_opt} is more involved and proceeds by establishing a Neyman--Pearson type lemma (Lemmas~\ref{lem:NeymanPearson_simple} and \ref{lem:NeymanPearson_complex}) showing that given a constraint on the `size' $\alpha$ that is sufficiently small, minimal subsampling enjoys maximal `power' $\beta$. To complete the argument, we show that any sequence of algorithms with size $\alpha$ remaining constant as $p \to \infty$ cannot have a subquadratic complexity, whilst Lemma~\ref{lem:subquad} attests that in contrast minimal subsampling does have subquadratic complexity under the assumptions of the theorem. Several auxiliary technical lemmas are collected in Section~\ref{sec:lemmas}

Our proofs Lemmas~\ref{lem:NeymanPearson_simple} and \ref{lem:NeymanPearson_complex} make use of the following bound on a quantity related to the ratio of the size to the power of minimal subsampling.

\begin{lem} \label{lem:alpha_minsub_bd}
Suppose $\psi \in \Psi$ has distribution for $M$ placing mass on $M$ and $M+1$.
Under the assumptions of Theorem~\ref{thm:optim},
\begin{equation*}
\frac{\alpha(\psi)}{\gamma_1^{M}} \leq \frac{2}{1-\rho}\frac{1}{M+1}.
\end{equation*}
\end{lem}
\begin{proof}
We have
\begin{align*}
\frac{\alpha(\psi)}{\gamma_1^M} \leq \frac{1}{p^2}\sum_{j,k} (\gamma_{jk}/\gamma_1)^{M} \leq \frac{c_0}{p}+ \sum_{i=0}^{n\gamma_1-1} \Big(\frac{i}{n\gamma_1}\Big)^Mf_n(i/n).
\end{align*}
Now the sum on the RHS is maximised over $f_n$ obeying constraints (A1) and (A2) in the following way. If $\rho\gamma_1n > \gamma_1n -1$ then $f_n$ places all available mass on $\gamma_1-1/n$. Otherwise $f_n$ should be as close to constant as possible on $\ceil{\rho\gamma_1n}/n,\ldots,(\gamma_1 n-1)/n$, and zero below $\ceil{\rho\gamma_1n}/n$.
In both cases it can be seen that
\begin{equation*}
\sum_{i=0}^{n\gamma_1-1} \Big(\frac{i}{n\gamma_1}\Big)^Mf_n(i/n) \leq \frac{2}{1-\rho} \int_{(1+\rho)/2}^1 x^M dx \leq \frac{2}{1-\rho}\frac{1}{M+1}.
\end{equation*}
\end{proof}
The following Neyman--Pearson-type lemma considers only non-randomised algorithms in $\Xi$. In Lemma~\ref{lem:NeymanPearson_complex} we extend this result to randomised algorithms.

\begin{lem} \label{lem:NeymanPearson_simple}
Let $\Xi_0$ be the set of $\xi \in \Xi$ such that $\xi_M$ places mass only on a single $M$, so the subsample size is not randomised.
There exists an $\alpha_0$ independent of $n$ such that for all $\alpha' \leq \alpha_0$, we have
\[
\sup_{\psi \in \Psi: \alpha(\psi)\leq \alpha'}\beta(\psi) = \sup_{\xi \in \Xi_0: \alpha(\xi) \leq \alpha'} \beta(\xi).
\]
Moreover the suprema are achieved.
\end{lem}
\begin{proof}
Each $\xi \in \Xi_0$ is parametrised by its close pairs threshold $\tau$ and subsample size $M$. Given a $\xi \in \Xi_0$ with parameter values $\tau$ and $M$ we compute $\alpha(\xi)$ as follows. Note that by replacing the threshold $\tau$ by $\tau/2$, we may assume that $\mb X$ and $\mb Z$ have entries in $\{-1/2, 1/2\}$. Thus $\mb X_j- \mb Z_k$ has components in $\{-1, 0, 1\}$. Let $J_{jk}$ be the number of non-zero components of $(X_{i_m j}-Z_{i_m k})_{m=1}^M$. Then $J_{jk}\sim \mathrm{Binom}(M, 1-\gamma_{jk})$. Thus
\begin{align*}
\pr\bigg( \abs{\sum_{m=1}^M D_m (X_{i_mj}-Z_{i_mk})} \leq \tau \bigg) = \pr(J_{jk}=0)+ \sum_{r=1}^M \pr\bigg( \abs{\sum_{m=1}^r D_m} \leq \tau \bigg) \pr(J_{jk}=r),
\end{align*}
noting that $D_m \eqdist -D_m$.
By Lemma~\ref{lem:dist_ineq} we know there exists an $a > 0$ such that for all $\tau \leq a \sqrt{M}$
the RHS is bounded below by
\begin{align} \label{eq:lemma9bd}
\gamma_{jk}^M + \sum_{r=r_0}^M \frac{c_1\tau}{\sqrt{r}}\binom{M}{r} \gamma_{jk}^{M-r}(1-\gamma_{jk})^r
\end{align}
for $M$ sufficiently large. Here the constants $a, c_1>0$ and $r_0 \in \N$ depend only on $F$.

Consider $\tau > a\sqrt{M}$. In this case, for $r \leq M$ sufficiently large we have by Lemma~\ref{lem:dist_ineq}
\begin{align*}
\pr\bigg( \abs{\sum_{m=1}^r D_m} \leq \tau \bigg) &\geq \pr\bigg( \abs{\sum_{m=1}^r D_m} \leq a\sqrt{r} \bigg) \geq c_1 a.
\end{align*}
However then for $M$ sufficiently large,
\[
\pr(J_{jk}=0)+ \sum_{r=1}^M \pr\bigg( \abs{\sum_{m=1}^r D_m} \leq \tau \bigg) \pr(J_{jk}=r) \geq c_1 a/2,
\]
so $\alpha(\xi) \geq c_1 a/2$. Note also that we must have $\alpha_0 \geq \alpha(\xi) \geq \gamma_l^M$, so $M \geq \log(\alpha_0) / \log(\gamma_l)$. Thus by choosing $0<\alpha_0 < c_1 a/2$ sufficiently small, we can rule out $\tau > a\sqrt{M}$ and so we henceforth assume that $\tau \leq a\sqrt{M}$, and that $M$ is sufficiently large such that \eqref{eq:lemma9bd} holds for all $(j,k)$.

We have
\begin{align} \label{eq:alpha_bd}
\alpha(\xi) \geq \frac{1}{p^2} \sum_{j,k} \bigg\{\gamma_{jk}^M + \tau \sum_{r=r_0}^M \frac{c_1}{\sqrt{r}} \binom{M}{r}\gamma_{jk}^{M-r}(1-\gamma_{jk})^r\bigg\}.
\end{align}
Similarly we have
\begin{align} \label{eq:beta_bd}
\beta(\xi) \leq \gamma_1^M + \tau\sum_{r=1}^M \frac{c_2}{\sqrt{r}}\binom{M}{r}\gamma_1^{M-r}(1-\gamma_1)^r.
\end{align}
Now substituting the upper bound on $\tau$ implied by \eqref{eq:alpha_bd} into \eqref{eq:beta_bd}, we get
\begin{align*}
\beta(\xi)\leq \gamma_1^M +Q_M \bigg(\alpha(\xi) - \frac{1}{p^2}\sum_{j,k}\gamma_{jk}^M\bigg)
\end{align*}
where
\begin{align*}
Q_M=\frac{c_2\sum_{r=1}^M r^{-1/2}\binom{M}{r}\gamma_1^{M-r}(1-\gamma_1)^r}{c_1 p^{-2}\sum_{j,k}\sum_{r=r_0} r^{-1/2}\binom{M}{r}\gamma_{jk}^{M-r}(1-\gamma_{jk})^r}.
\end{align*}
Now by Lemma~\ref{lem:binomailsqrt}, for $M$ sufficiently large and some constant $Q$ we have
\[
Q_M \leq Q\frac{\sqrt{1-\gamma_1}}{\sum_{j,k}\sqrt{1-\gamma_{jk}}/p^2} \leq Q.
\]
Thus
\begin{equation} \label{eq:beta_xi_bd}
\beta(\xi) \leq \gamma_1^M + Q \bigg(\alpha(\xi) - \frac{1}{p^2}\sum_{j,k}\gamma_{jk}^M\bigg)
\end{equation}
for all $M$ sufficiently large. Now given $\alpha_0$, let $M_0$ be such that
\[
 \frac{1}{p^2} \sum_{j,k}\gamma_{jk}^{M_0} \geq \alpha_0 \geq \frac{1}{p^2} \sum_{j,k}\gamma_{jk}^{M_0+1}.
\]
Consider the minimal subsampling algorithm $\psi$ that chooses subsample size as either $M_0$ or $M_0+1$ with probabilities $b$ and $1-b$ such that
\[
\alpha(\psi)=\frac{1}{p^2}\sum_{j,k}\{b\gamma_{jk}^{M_0} + (1-b)\gamma_{jk}^{M_0+1}\}= \alpha_0.
\]
Then we have $\beta(\psi) = b\gamma_1^{M_0} + (1-b)\gamma_1^{M_0+1}$. Now suppose $\xi \in \Xi_0$ has $\alpha(\xi) \leq \alpha_0$. Then in particular $M \geq M_0+1$. We first examine the case where $M=M_0+1$. Then
\begin{align*}
\frac{1}{\gamma_1^{M_0}}\{\beta(\psi) - \beta(\xi)\} &\geq b + (1-b)\gamma_1 -\gamma_1 - \frac{Q}{\gamma_1^{M_0} } \bigg( \alpha_0 - \frac{1}{p^2}\sum_{j,k}\gamma_{j,k}^{M_0+1}\bigg)\\
&= b + (1-b)\gamma_1 -\gamma_1 - \frac{aQ}{\gamma_1^{M_0} } \frac{1}{p^2}{\sum_{j,k}(\gamma_{j,k}^{M_0} - \gamma_{j,k}^{M_0+1})} \\
&\geq b\bigg((1-\gamma_u) - \frac{2Q}{1-\rho} \frac{1}{M_0+1}\bigg),
\end{align*}
using Lemma~\ref{lem:alpha_minsub_bd} in the final line. Note this is non-negative for $M_0$ sufficiently large.
When $M \geq M_0+2$ we instead have
\begin{align*}
\frac{\beta(\xi)}{\beta(\psi)} &\leq \frac{\beta(\xi)}{\gamma_1^{M_0+1}} \leq \gamma_1 + \frac{2Q}{\gamma_1(1-\rho)} \frac{1}{M_0+1} \leq \gamma_u + \frac{2Q}{\gamma_l(1-\rho)} \frac{1}{M_0+1} <1
\end{align*}
for $M_0$ sufficiently large.
Recall that by making $\alpha_0$ sufficiently small, we can force $M_0$ to be arbitrarily large. Thus the result is proved.
\end{proof}
\begin{lem} \label{lem:NeymanPearson_complex}
There exists an $\alpha_0$ independent of $n$ such that for all $\alpha' \leq \alpha_0$, we have
\[
\sup_{\psi \in \Psi:\alpha(\psi)\leq \alpha'} \beta(\psi) =\sup_{\xi \in \Xi:\alpha(\xi)\leq \alpha'} \beta(\xi).
\]
Moreover the suprema are achieved.
\end{lem}
\begin{proof}
With a slight abuse of notation, write $\xi(M', \tau')$ for the element of $\xi \in \Xi$ that fixes $M=M'$ and $\tau=\tau'$.
Using the notation of Lemma~\ref{lem:NeymanPearson_simple}, define function $f:[0,1] \to [0,1]$ by
\[
f(\alpha')=\sup_{\xi \in \Xi_0:\, \alpha(\xi) \leq \alpha'} \beta(\xi).
\]
Note that for $\xi \in \Xi$ we have
\begin{equation} \label{eq:beta_f_bd}
\beta(\xi) \leq \E_{M\sim \xi_M} f[\alpha\{\xi(M,\xi_\tau)\}].
\end{equation}
Now by Lemma~\ref{lem:NeymanPearson_simple} we know there exists $\alpha_0$ (depending on $F$) such that on $[0,\alpha_0]$, $f$ is the linear interpolation of points
\[
\bigg(\frac{1}{p^2}\sum_{j,k} \gamma_{j,k}^M,\, \gamma_1^M \bigg)_{M=1}^{\infty}.
\]
We claim that $f$ is concave on $[0,\alpha_0]$. Indeed, it suffices to show that the slopes of the successive linear interpolants are decreasing in this region, or equivalently that their reciprocals are increasing. We have
\begin{align} \label{eq:deriv_interp}
\frac{1}{p^2}\sum_{j,k} \frac{\gamma_{jk}^{M+1} - \gamma_{jk}^M}{\gamma_1^{M+1} - \gamma_1^M} = \frac{1}{p^2}\sum_{j,k} \bigg(\frac{\gamma_{j,k}}{\gamma_1} \bigg)^M\frac{\gamma_{jk}-1}{\gamma_1-1}
\end{align}
which increases as $M$ decreases, thus proving the claim.

Note also that the RHS of \eqref{eq:deriv_interp} is at most $\alpha(\psi)/\{(1-\gamma_u)\gamma_1^M\}$ when $\psi$ has subsample size fixed at $M$. Thus by Lemma~\ref{lem:alpha_minsub_bd} we see the derivatives of the linear interpolants approach infinity as they get closer to the origin. This implies the existence of an $0<\alpha_1<\alpha_0$ such that $-\sup\big(\partial(-f)(\alpha_1)\big) \geq \{1-f(\alpha_1)\}/(\alpha_0-\alpha_1)$, where $\partial(-f)(\alpha_1)$ denotes the subdifferential of the function $-f$ at $\alpha_1$. We may therefore invoke Lemma~\ref{lem:concave_ineq} to conclude that for $\xi$ with $\alpha(\xi) \leq \alpha_1$
\begin{align*}
 \E_{M \sim \xi_M} f[\alpha\{\xi(M,\xi_\tau)\}] \leq f[\E_{M \sim \xi_M} \alpha\{\xi(M,\xi_\tau)\}] = f(\alpha(\xi)) \leq f(\alpha_1)= \max_{\psi \in \Psi:\alpha(\psi)\leq \alpha_1} \beta(\psi).
\end{align*}
Combining with \eqref{eq:beta_f_bd} gives the result.
\end{proof}
The next lemma establishes subquadratic complexity of minimal subsampling.
\begin{lem} \label{lem:subquad}
Under the assumptions of Theorem~\ref{thm:optim}, we have
$\inf_{\psi \in \Psi(\eta)} T(\psi) /(np^2) \to 0$.
\end{lem}
\begin{proof}
Let $\psi \in \Psi$ be such that $\psi_M$ places all mass on $M$.
We have that $\beta(\psi) = \gamma_1^M$. Thus using the inequality $-x \leq \log(1-x)$ for $x \in (0,1)$, we have
\[
\psi_L \leq -\gamma_1^{-M}\log(1-\eta).
\]
Lemma~\ref{lem:alpha_minsub_bd} gives an upper bound on $\psi_L\E_{\psi} E_1 $.
%By Lemma~\ref{lem:close_pairs}, $\E_\psi V =\mathcal{O}(p \log(p))$.
Note that $\E_\psi V =\mathcal{O}(p \log(p))$.
Thus ignoring constant factors, we have
\[
T(\psi)/(np^2) \leq \frac{M + \log(p)}{\gamma_1^M np} + \frac{1}{M+1}.
\]
Taking $M=\floor{\log(1/\sqrt{p})/\log(\gamma_1)}$ then ensures $T(\psi)/(np^2) \to 0$.
\end{proof}

\begin{lem} \label{lem:dense}
Let $\xi \in \Xi_{\text{dense}}$. There exists $c>0$ and $n_0 \in \N$ such that for all $n \geq n_0$,
\[
\inf_{\xi \in \Xi_{\text{\text{dense}}}}T(\xi)/(np^2) > c.
\]
\end{lem}
\begin{proof}
Each $\xi \in \Xi_{\text{dense}}$ is parametrised by its close pairs threshold $\tau$. Given a $\xi \in \Xi_{\text{dense}}(F)$ with close pairs threshold $\tau$ we compute $\alpha(\xi)$ as follows. Similarly to Lemma~\ref{lem:NeymanPearson_simple} we may assume without loss of generality that $\mb X$ and $\mb Z$ have entries in $\{-1/2, 1/2\}$ so $\mb X_j- \mb Z_k$ has components in $\{-1, 0, 1\}$. Since $R_i \eqdist -R_i$ as $F \in \mathcal{F}$, we have
\begin{align*}
\pr\bigg( \abs{\sum_{i=1}^n R_i (X_{ij}-Z_{ik})} \leq \tau \bigg) = \pr\bigg( \abs{\sum_{i=1}^{n(1-\gamma_{jk})} R_i} \leq \tau \bigg).
\end{align*}
We now use Lemma~\ref{lem:dist_ineq}. For $n(1-\gamma_u)$ sufficiently large, when $\tau \leq a \sqrt{n}$ the RHS is bounded below by
\begin{align*}
\frac{c_1\tau}{\sqrt{ n(1-\gamma_{jk})}}.
\end{align*}
Here constant  $a, c_1 > 0$ also depend only on $F$.
Thus
\begin{align} \label{eq:alpha_bd_dense}
\alpha(\xi) \geq \frac{1}{p^2} \sum_{j,k} \frac{c_1\tau}{\sqrt{ n(1-\gamma_{jk}})}
\geq c_1 \tau / \sqrt{n}.
\end{align}
Similarly we have
\begin{align} \label{eq:beta_bd_dense}
\beta(\xi) \leq \frac{c_2\tau}{\sqrt{n (1-\gamma_1)}}.
\end{align}
Note that from \eqref{eq:alpha_bd_dense}, when $\tau > a \sqrt{n}$ we have $\alpha(\xi) \geq c_1 a$. Thus from \eqref{eq:complexity_dense} we know there exists $n_0$ such that for all $n \geq n_0$, we have
\begin{equation} \label{eq:dense_tau_large}
\inf_{\xi \in \Xi_{\text{dense}}(\eta) : \xi_\tau > a \sqrt{n}} T(\xi)/(np^2) \geq \inf_{\xi \in \Xi_{\text{dense}}(\eta) : \xi_\tau > a \sqrt{n}} \xi_L \alpha(\xi) \geq  \xi_L c_1a > 0.
\end{equation}
We therefore need only consider the case where $\tau \leq a \sqrt{n}$ and where $\alpha(\xi) \to 0$.

Substituting the upper bound on $\tau$ implied by \eqref{eq:alpha_bd_dense} into \eqref{eq:beta_bd_dense}, we get
\begin{align*}
\beta(\xi)\leq \alpha(\xi) \frac{c_2}{c_1\sqrt{1-\gamma_u}}.
\end{align*}
Note that then
\begin{align*}
\xi_L \geq \frac{\log(1-\eta)}{\log\{ 1 - \alpha(\xi)c_2 /(c_1\sqrt{1-\gamma_u})\}} \geq c_3\frac{\log\big(1/1-\eta\big)}{\alpha(\xi)}
\end{align*}
for some $c_3>0$ provided $\alpha(\xi) < 1/2$ say. However this gives us
\[
\inf_{\xi \in \Xi_{\text{dense}}(\eta) : \xi_\tau \leq a \sqrt{n}} T(\xi)/(np^2) \geq \inf_{\xi \in \Xi_{\text{dense}}(\eta) : \xi_\tau \leq a \sqrt{n}} \xi_L \alpha(\xi) \geq  \min\{1/2, c_3\log\big(1/1-\eta\big)\} > 0.
\]
Combined with \eqref{eq:dense_tau_large} this give the result.
\end{proof}

With the previous lemmas in place, we are in a position to prove \eqref{eq:res_opt} of Theorem~\ref{thm:optim}.

\subsubsection*{Proof of Theorem~\ref{thm:optim}}

The proofs of \eqref{eq:res_subquad} and \eqref{eq:res_dense} are contained in Lemmas~\ref{lem:subquad} and~\ref{lem:dense} respectively. To show \eqref{eq:res_opt} we argue as follows.
Given $F$ and $\eta$, suppose for contradiction that there exists a sequence $\xi^{(1)}, \xi^{(2)},\ldots$ and $n_1 < n_2< \cdots$ such that (making the dependence on $n$ of the computational time explicit)
\[
\inf_{\psi \in \Psi(\eta)}T^{(n_k)}(\psi) > T^{(n_k)}(\xi^{(k)})
\]
for all $k$. By Lemma~\ref{lem:subquad}, we must have $T^{(n_k)}(\xi^{(k)})/(np^2) \to 0$. This implies that $\alpha(\xi^{(k)}) \to 0$. By Lemma~\ref{lem:NeymanPearson_complex}, we know that for $k$ sufficiently large
\[
\sup_{\psi \in \Psi:\alpha(\psi) = \alpha(\xi^{(k)})} \beta(\psi) \geq \beta(\xi^{(k)}).
\]
Let $\psi^{(k)}$ be the maximiser of the LHS. In order for $T^{(n_k)}(\psi^{(k)}) > T^{(n_k)}(\xi^{(k)})$, it must be the case that $\E_{M\sim\psi^{(k)}_M} M > \E_{M \sim \xi^{(k)}_M} M$.
However we claim that $\xi=\psi^{(k)}$ minimises $\E_{M\sim \xi_M} M$ among all $\xi \in \Xi$ with $\alpha(\xi) \leq \alpha(\xi^{(k)})=:\alpha_0$, which gives a contradiction and completes the proof.
%Indeed, given any minimising $\xi$ we may assume the threshold $\tau \equiv 0$ since this reduces $\alpha(\xi)$ without changing the expectation of $M$: $\xi$ can then be modified to reduce $M$ whilst still remaining feasible such that $\alpha(\xi) \leq \alpha(\xi^{(k)})$.
Let $f$ be the function that linearly interpolates the points
\[
\bigg(\frac{1}{p^2} \sum_{j,k} \gamma_{j,k}^M, \, M\bigg)_{M=1}^\infty.
\]
Note that $f$ is decreasing.
By considering the inverse of $f$ it is clear that $f$ is convex. With a slight abuse of notation, write $\xi(M, \tau)$ for the element of $\xi \in \Xi$ such that $\xi_M$ places all mass on $M$ and $\xi_\tau=\tau$. Note that
\begin{equation*}
\E_{M \sim \xi_M} M = \E_{M\sim \xi_M} f[\alpha\{\xi(M, 0)\}] \geq \E_{M \sim \xi_M} f[\alpha\{\xi(M, \xi_\tau)\}].
\end{equation*}
Now suppose $\xi$ has $\alpha(\xi) \leq \alpha_0$. Then from the above and Jensen's inequality,
\[
\E_{M \sim \xi_M} M \geq f\big(\E_{M \sim \xi_M} \alpha(\xi(M, \xi_\tau))\big) \geq f(\alpha_0) = \E_{M\sim\psi^{(k)}_M} M.  \qedhere
\]

%\subsection{Proof of Lemma~\ref{lem:close_pairs}}
\subsection*{Proof of Theorem~\ref{thm:minsamp}}
First note that from \eqref{eq:eta} we have $L \leq \log(1-\eta')/\log(1-\gamma^M) + 1$. Then using the inequality $\log(1-x) \leq -x$ for $x \in (0,1)$, we have
\[
L \leq \frac{\log\{1/(1-\eta')\} + 1}{\gamma^M}.
\]
Note that from the definition of $\gamma_0$ we have $\gamma^{-M} = p^{\log(\gamma)/\log(\gamma_0)}$. We then see that
\begin{align*}
\gamma^{-M}\E(E_1) &= \gamma^{-M}\sum_{j,k} \gamma_{jk}^M \\
&\leq \gamma^{-M}\bigg(\sum_{j,k : \gamma_{jk} > \gamma } \gamma_{jk}^M + \sum_{j,k : \gamma_0 < \gamma_{jk} \leq \gamma } \gamma_{jk}^M + \sum_{j,k:\gamma_{jk} \leq \gamma_0} \gamma_{jk}^M \bigg)\\
&\leq c_1p\gamma^{-M} + c_2p^{1 + \log(\gamma)/\log(\gamma_0)} + p^2\gamma_0^M\gamma^{-M} \\
&\leq (c_1 + c_2 + 1)p^{1 + \log(\gamma)/\log(\gamma_0)}.
\end{align*}
Collecting together the terms in \eqref{eq:complexity} we have
\begin{align*}
C(M,L) \leq np + [\log\{1/(1-\eta')\} + 1][\log(p)\{1 + 1/\log(\gamma_0^{-1})\}  + n(c_1 + c_2 + 1)]p^{1 + \log(\gamma)/\log(\gamma_0)}
\end{align*}
from which the result easily follows.
\subsection*{Proof of Proposition~\ref{prop:Pareto}}
Let $\eta^* = \eta(M^*, L)$. Note that in order for $\eta(M', L') \geq \eta^*$ it must be the case that $L' \geq \log(1-\eta^*)/\log(1-\gamma^{M'})$. Therefore
\begin{align}
C(M',L') -np &\geq \frac{\log(1-\eta^*)}{\log(1-\gamma^{M'})}\bigg(M'p + p\log(p) + n\sum_{j,k}\gamma_{jk}^{M'}\bigg) \notag\\
&\geq \min_{M \in \N} \frac{\log(1-\eta^*)}{\log(1-\gamma^{M})}\bigg(Mp + p\log(p) + n\sum_{j,k}\gamma_{jk}^{M}\bigg) \label{eq:min_M}\\
&= \frac{\log(1-\eta^*)}{\log(1-\gamma^{M^*})}\bigg(M^*p + p\log(p) + n\sum_{j,k}\gamma_{jk}^{M^*}\bigg) = C(M^*, L). \notag
\end{align}
Moreover, the inequality leading to \eqref{eq:min_M} is strict if $M^*$ is the unique minimiser and $M' \neq M^*$.
\subsection*{Technical lemmas} \label{sec:lemmas}
\begin{lem} \label{lem:dist_ineq}
Let $F \in \mathcal{F}$ and suppose $(R_i)_{i=1}^\infty$ is an i.i.d.\ sequence with $R_i \sim F$.

Then for all $a > 0$, there exists $c_1, c_2 >0$ and $l_0 \in \mathbb{N}$ such that for all $l \geq l_0$ and $0 \leq \tau \leq a \sqrt{l}$ we have
\[
\frac{c_1\tau}{\sqrt{l}} \leq \pr\Big(\big| \sum_{i=1}^l R_i \big| \leq \tau\Big) \leq \frac{c_2\tau}{\sqrt{l}}.
\]
\end{lem}
\begin{proof}
Let $f_l$ be the density of $\sum_{i=1}^l R_i/ \sqrt{l}$. Note that as $\E(|R_1|^3) <\infty$, we must have $\E(R_1^2) <\infty$, so we may assume without loss of generality that $\E(R_1^2)=1$. Then by Theorem 3 of \citet{Petrov1964} we have that for sufficiently large $l$,
\begin{equation} \label{eq:density_clt}
|f_l(t)- \phi(t)|\leq \frac{c}{\sqrt{l}(1+|t|^3)}.
\end{equation}
Here $c$ is a constant and $\phi(t) = e^{-t^2/2}/\sqrt{2\pi}$ is the standard normal density. Now by the mean value theorem, we have
\[
2\inf_{0\leq t\leq \tau/\sqrt{l}}\{f_l(t)\} \frac{\tau}{\sqrt{l}} \leq \pr\Big(\big| \sum_{i=1}^l R_i \big|/\sqrt{l} \leq \tau/\sqrt{l}\Big) \leq 2\sup_{0\leq t\leq \tau/\sqrt{l}}\{f_l(t)\} \frac{\tau}{\sqrt{l}}.
\]
Thus from \eqref{eq:density_clt}, for $l$ sufficiently large we have
\begin{align*}
\pr\Big(\big| \sum_{i=1}^l R_i \big| \leq \tau\Big) \geq \frac{\tau}{\sqrt{l}} \bigg(\frac{\sqrt{2}}{\sqrt{\pi}}\exp\{-\tau^2 / (2l)\}  -\frac{2c}{\sqrt{l}}\bigg).
\end{align*}
Note that for $a>0$ and $l$ sufficiently large we have $\sqrt{2/\pi} e^{-a^2/2} > 2c/\sqrt{l}$, whence
\[
\pr\Big(\big| \sum_{i=1}^l R_i \big| \leq \tau \Big) \geq \frac{c_1\tau}{\sqrt{l}}
\]
for $0\leq \tau \leq a \sqrt{l}$, some $c_1 > 0$. A similar argument yields the upper bound in the final result.
\end{proof}

\begin{lem} \label{lem:binomailsqrt}
Suppose $\gamma \in [0,1)$.
For all $M \in \N$ we have
\begin{equation} \label{eq:binom_upper}
\sum_{r=1}^M \frac{1}{\sqrt{r}} \binom{M}{r} (1-\gamma)^r \gamma^{M-r}\leq \frac{\sqrt{2}}{\sqrt{(1-\gamma)M}}.
\end{equation}
Given $r_0 \in \N$ and $\gamma \in [0,1)$, there exists $c>0$ and $M_0 \in \N$ such that for all $M \geq M_0$ we have
\begin{equation} \label{eq:binom_lower}
\sum_{r=r_0}^M \frac{1}{\sqrt{r}} \binom{M}{r} (1-\gamma)^r \gamma^{M-r} \geq \frac{c}{\sqrt{(1-\gamma)M}}.
\end{equation}
\end{lem}
\begin{proof}
First we show the upper bound \eqref{eq:binom_upper}. Let $J \sim \mathrm{Binomial}(M,1-\gamma)$.
\begin{align*}
\sum_{r=1}^M \frac{1}{\sqrt{r}} \binom{M}{r} (1-\gamma)^r \gamma^{M-r} & \leq \sqrt{2} \sum_{r=1}^M \frac{1}{\sqrt{r+1}}\binom{M}{r}(1-\gamma)^r\gamma^{M-r} \\
&\leq \sqrt{2}\E(1/\sqrt{J+1}).
\end{align*}
Next, by Jensen's inequality we have $\E(1/\sqrt{J+1}) \leq \sqrt{\E\{1/(J+1)\}}$. We now compute $\E\{1/(J+1)\}$ as follows.
\begin{align*}
\E\bigg(\frac{1}{J+1}\bigg) &= \sum_{r=0}^M \frac{1}{r+1} \binom{M}{r} (1-\gamma)^r \gamma^{M-r} \\
&= \frac{1}{M+1} \sum_{r=0}^M \binom{M+1}{r+1} (1-\gamma)^r \gamma^{M-r} \\
&= \frac{1}{(1-\gamma)(M+1)} \sum_{r=0}^M \binom{M+1}{r+1} (1-\gamma)^{r+1} \gamma^{M-r} \\
&= \frac{1-\gamma^{M+1}}{(1-\gamma)(M+1)} \leq \frac{1}{(1-\gamma)(M+1)}.
\end{align*}
Putting things together gives \eqref{eq:binom_upper}.

Turning now to \eqref{eq:binom_lower}, we see that the LHS equals
\[
\E(1/\sqrt{J}\ind_{\{J\geq r_0\}}) = \E(1/\sqrt{J}|J\geq r_0)\pr(J\geq r_0).
\]
By Jensen's inequality we have
\begin{align*}
\E(1/\sqrt{J}|J\geq r_0)\geq \frac{1}{\sqrt{\E(J|J \geq r_0)}}=\frac{\sqrt{\pr(J\geq r_0)}}{\sqrt{\E(J\ind_{\{J \geq r_0\}})}} \geq \frac{\sqrt{\pr(J\geq r_0)}}{\sqrt{(1-\gamma)M}}.
\end{align*}
But as $M \to \infty$, $\pr(J\geq r_0) \to 1$, which easily gives the result.
\end{proof}

\begin{lem}\label{lem:concave_ineq}
Let $f:[0,\infty) \to [0,1]$ be non-decreasing. Suppose there exists $0 < \alpha_1 < \alpha_0$ such that:
\begin{enumerate}[(i)]
\item $f$ is concave on $[0,\alpha_0]$;
\item $-\sup\big(\partial(-f)(\alpha_1)\big) \geq \{1-f(\alpha_1)\}/(\alpha_0-\alpha_1)$, where $\partial(-f)(\alpha_1)$ denotes the subdifferential of the function $-f$ at $\alpha_1$.
\end{enumerate}
Then if random variable $X$ has $\E(X) \leq \alpha_0$, then $f(\E X) \geq \E f(X)$.
\end{lem}
\begin{proof}
Write $m=-\sup\big(\partial(-f)(\alpha_1)\big)$
Let function $g:[0, \infty) \to [0, \infty)$ be defined as follows.
\begin{equation*}
g(x) = \begin{cases}
f(x) \qquad &\text{if } 0\leq x \leq \alpha_1 \\
f(\alpha_1) + m(x-\alpha_1) \qquad &\text{if } x > \alpha_1.
\end{cases}
\end{equation*}
Note that $g$ thus defined has $g(\alpha_0) \geq 1$. We see that $g$ is convex and $g \geq f$. Thus if $\E(X) \leq \alpha_1$, by Jensen's inequality we have
\[
f ( \E X) = g( \E X) \geq \E g(X) \geq \E f(X).
\]
\end{proof}

\section*{Appendix B}

\subsection*{Connection to LSH}

Minimal subsampling as considered in Algorithm \ref{alg:final_xyz} is closely related to the locality-sensitive hashing (LSH) framework: Define $h(j)=\mb R^T\matr{X}_j$ ($\mb R$ corresponds to the minimal subsampling projection) to be the hashing function and $\mathcal{H}$ to be the family of such functions, from which we sample uniformly. Then $\mathcal{H}$ is $(\gamma,c \gamma,p_1,p_2)$-sensitive, that is:
\begin{align*}
&\bullet \textrm{ if } \gamma_{jk} \geq \gamma \textrm{ then } \mathbb{P}(h(j)=h(k)) \geq p_1\\
&\bullet \textrm{ if } \gamma_{jk} \leq c\gamma \textrm{ then } \mathbb{P}(h(j)=h(k)) \leq p_2,
\end{align*}
where $0 < c < 1$. In the case of the minimal subsampling we have $p_1=\gamma^M$ and $p_2=\gamma^M c^M$. However, the typical LSH machinery cannot be applied directly to the equal pairs problem above. In our setting, we are not interested in preserving close pairs but rather the closest pairs.
%We are not interested in preserving close pairs, but more precisely we are intersted in preserving the closest pairs.
Theorem \ref{thm:optim}
establishes that the family $\mathcal{H}$ leads to the maximal ratio $p_1/p_2$ among all linear hashing families.

\section*{Appendix C}

\subsection*{Proof of Proposition \ref{lem:transform}}

\begin{proof}
\begin{align*}
\mathbb{P}(\sgn(Y_i)=\tilde{X}_{ij}\tilde{X}_{ik}) &= \frac{\sgn(Y_i)+1}{2}(g(X_{ij})g(X_{ik})+(1-g(X_{ij}))(1-g(X_{ik}))) \\
&+\frac{1-\sgn(Y_i)}{2} (g(X_{ij})(1-g(X_{ik}))+(1-g(X_{ij}))g(X_{ik}))\\
&=\frac{1}{2}+\frac{\sgn(Y_i)}{2}(1-2g(X_{ij}))(1-2g(X_{ik})).
\end{align*}
\end{proof}
\section*{Appendix D}

\subsection*{The unbiased transform and the sign transform}

\textbf{Proposition \ref{prop:unbiased}}
\begin{proof}
The equation
\begin{equation*}
\E[\tilde{X}_{ij}] = \mathbb{P}(\tilde{X}_{ij}=1)-\mathbb{P}(\tilde{X}_{ij}=-1) = X_{ij},
\end{equation*}
implies
\begin{equation*}
\mathbb{P}(\tilde{X}_{ij}=1) = \frac{X_{ij}+1}{2}
\end{equation*}
This uniquely determines the unbiased transform.
\end{proof}
Next we show two Lemmas that will be useful when proving Theorems~\ref{thm:unbiased} and~\ref{thm:signdist}.
\begin{lem}
\label{lem:unbiasedhelper}
Consider the setup of Theorem~\ref{thm:unbiased}.
%Assume $\varepsilon_i$ has a symmetric distribution around $0$ and follows a sub-exponential distribution. $X_{ij}$ has mean $0$ and lies in $[-1,1]$, also assume assumptions (B1)-(B3) hold.
Then there exists constants $ C_1^\varepsilon, C_2^\varepsilon > 0$ such that defining
\[
\alpha^u_{n,p} = \alpha^u_{n,p}(t) = \bigg(1+ \frac{t+\log(n C_1^\varepsilon)}{C^{\varepsilon}_2}\bigg)\sqrt{2\{t+\log(4p)\}/n},
\]
with probability at least $1-2\exp(-t)$ we have:
\begin{align*}
\frac{\sum_{i} Y_i  X_{ij^*} X_{ik^*} } { \| \mb Y \|_1} &\notin  \Big [-\frac{ m_2-\alpha^u_{n,p}}{m_1 +  m_\varepsilon +\alpha^u_{n,p}},\frac{ m_2-\alpha^u_{n,p}}{m_1 +  m_\varepsilon +\alpha^u_{n,p}} \Big] \\
\frac{\sum_{i=1}^n Y_i X_{ij} X_{ik}}{\| \mb Y \|_1}  &\in \Big [ - \frac{ m_2 (1- r_u) +\alpha^u_{n,p}}{  m_1 -\alpha^u_{n,p}}, \frac{ m_2 (1 - r_u) +\alpha^u_{n,p}}{  m_1 -\alpha^u_{n,p}} \Big ] \,\, \forall (j,k) \neq (j^*,k^*).
\end{align*}
\end{lem}
\begin{proof}
First we consider a capped version of $\varepsilon$:
\begin{equation*}
\varepsilon_i^\prime = \begin{cases}
\varepsilon_{i} \textrm{ if } |\varepsilon_{i}| \leq \sigma\\
\sigma \sgn(\varepsilon_i) \textrm{ otherwise},
\end{cases}
\end{equation*}
where $\sigma$ is to be chosen later. We may apply Hoeffding's inequality to these bounded variables. We have to bound two terms:
\[
\frac{\sum_{i=1}^n Y_i X_{ij^*} X_{ik^*}}{ \| \mb Y \|_1 } \textrm{ from below and } \frac{\sum_{i=1}^n Y_i X_{ij} X_{ik}}{ \| \mb Y \|_1 } \textrm{ from above, for } (j,k) \neq (j^*,k^*).
\]
Schematically the first term can be dealt with in the following way:
\begin{align}
\label{ineq:useful}
\mathbb{P}\Big(\frac{A+B}{C+D} \geq \frac{a+b}{c+d} \Big) &\geq 1-\mathbb{P}(A \leq a)-\mathbb{P}(B \leq b)- \mathbb{P}(C \geq c)-\mathbb{P}(D \geq d)
\end{align}
where
\[
A+B=\sum_{i=1}^n (X_{ij^*} X_{ik^*})^2 +\varepsilon^\prime_{i} X_{ij^*} X_{ik^*} \,\, \textrm{ and } \,\, C+D = \sum_{i=1}^n |X_{ij^*}X_{ik^*}+\varepsilon^\prime_i |.
\]
We deal with each term individually. Using Hoeffding's inequality we get:
\begin{itemize}
\item[$A: $] $\mathbb{P} \Big( \sum_{i=1}^p ( X_{ij^*}X_{ik^*} )^2 \leq  n m_2-\delta \Big) \leq \exp( - \delta^2 /  2n )
)$
\item[$B: $] $\mathbb{P} \Big( \sum_{i=1}^n \varepsilon^\prime_i X_{ij^*}X_{ik^*} \leq -\kappa \Big) \leq \exp( - \kappa^2 /  2 n \sigma^2 )
$
\item[$C: $] $\mathbb{P} \Big( \sum_{i=1}^n | X_{ij^*}X_{ik^*} | \geq  n m_1+\delta \Big) \leq \exp( - \delta^2 /  2n ) $
\item[$D: $] $\mathbb{P} \Big( \sum_{i=1}^n |\varepsilon^\prime_i| \geq n m_\varepsilon+ \kappa \Big)\leq \exp( -  2\kappa^2 /  n \sigma^2)$.
\end{itemize}
This gives us a bound of the interaction strength of the true interaction pair:
\begin{align*}
\mathbb{P} \Big(\frac{\sum_{i} Y_i  X_{ij^*} X_{ik^*}  } { \| \mb Y \|_1} &\geq \frac{n m_2-\delta-\kappa}{nm_1 + n m_\varepsilon +\delta+\kappa}  \Big)\\
&\geq 1- \exp( - \delta^2 /  2n) - \exp( - \delta^2 /  2n) \\
&-\exp( - \kappa^2 /  2 n \sigma^2)-\exp( - \kappa^2 /  2 n \sigma^2)
\end{align*}
Similarly we can treat the interaction strength of the non interacting pairs:
\begin{itemize}
\item[$A: $] Here we use assumption $(B1)$:
\[
m_2 ( r_u-1) \leq \E[X_{ij^*} X_{ik^*} X_{im} X_{io} ] \leq m_2 (1- r_u).
\]
Hence, $\mathbb{P} \Big( \sum_{i=1}^n X_{ij^*} X_{ik^*} X_{ij} X_{ik} \geq n m_2 (1- r_u) + \delta \Big)  \geq \exp( -   \delta / 2 n ).
$
\end{itemize}
For the rest we run the same bounds as before (using $|X_{ij^*}X_{ik^*}+\varepsilon^\prime_i | \geq |X_{ij^*}X_{ik^*}|+\varepsilon^\prime_i$). This yields the bound
\begin{align*}
\mathbb{P} \Big( \frac{\sum_{i=1}^n Y_i X_{ij} X_{ik}}{\| \mb Y \|_1}  &\leq \frac{n m_2 (1- r_u) +\delta + \kappa}{ n m_1 -\delta -\kappa} \Big) \\
&\geq 1- \exp( - \delta^2 /  2n ) - \exp( - \delta^2 /  2n) \\
&-\exp( - \kappa^2 /  2 n \sigma^2 )-\exp( - \kappa^2 /  2 n \sigma^2)
\end{align*}
The above inequality needs to hold for all at most $p^2$ pairs that are not interactions, so that we effectively multiply the exponential terms with $p^2$. Another factor of $2$ is multiplied in for the negative sign, as the fraction also has to be bounded away from $-1$. In total we thus have:
\begin{align*}
\frac{\sum_{i=1}^n Y_i  X_{ij^*} X_{ik^*} } { \| \mb Y \|_1} &\notin  \Big [-\frac{n m_2-\delta-\kappa}{nm_1 + n m_\varepsilon +\delta+\kappa},\frac{n m_2-\delta-\kappa}{nm_1 + n m_\varepsilon +\delta+\kappa} \Big] \\
\frac{\sum_{i=1}^n Y_i X_{ij} X_{ik}}{\| \mb Y \|_1}  &\in \Big [ - \frac{n m_2 (1- r_u) +\delta + \kappa}{ n m_1 -\delta -\kappa}, \frac{n m_2 (1- r_u) +\delta + \kappa}{ n m_1 -\delta -\kappa} \Big ] \,\, \forall (m,o) \neq (j,l)\\
\textrm{ with probability at least }& 1- \exp( - \delta^2 /  2n ) - \exp( - \delta^2 /  2n )-\exp( - \kappa^2 /  2 n \sigma^2 )-\exp( - \kappa^2 /  2 n \sigma^2).
\end{align*}
Finally, let $\sigma \geq 1$, then we have to set $\delta$ and $\kappa$ so that the probability is bigger than $1-\exp(-t)$.   This gives:
\[
\exp(-t) = 4 p \exp(-\delta^2/2n) \textrm{ and } \exp(-t) = 4 p \exp(-\kappa^2/2n \sigma^2).
\]
This gives
\[
\delta = \sqrt{2n (t+\log(4p))} \textrm{ and } \kappa = \sqrt{2n \sigma^2  (t+\log(4p))}.
\]
Thus for $\alpha^u_{n,p} = \frac{\sqrt{2(t+\log(4p))(1+\sigma^2 )}}{\sqrt{n}}$,
\begin{align*}
\frac{\sum_{i} Y_i  X_{ij^*} X_{ik^*} } { \| \mb Y \|_1} &\notin  \Big [-\frac{ m_2-\alpha^u_{n,p}}{m_1 +  m_\varepsilon +\alpha^u_{n,p}},\frac{ m_2-\alpha^u_{n,p}}{m_1 +  m_\varepsilon +\alpha^u_{n,p}} \Big] \\
\frac{\sum_{i=1}^n Y_i X_{ij} X_{ik}}{\| \mb Y \|_1}  &\in \Big [ - \frac{ m_2 (1- r_u) +\alpha^u_{n,p}}{  m_1 -\alpha^u_{n,p}}, \frac{ m_2 (1- r_u) +\alpha^u_{n,p}}{  m_1 -\alpha^u_{n,p}} \Big ] \,\, \forall (j,k) \neq (j^*,k^*)\\
\textrm{ with probability at least }& 1- \exp(-t).
\end{align*}
Now we extend this result to the case of unbounded errors, that is we now assume that with high probability $\varepsilon_i$ are bounded:
\begin{equation*}
\mathbb{P}( \varepsilon_i = \varepsilon^\prime_{i}, \,\, \forall \, i  )=1-\exp(-t).
\end{equation*}
Here we used the sub-exponential tail behavior of $\varepsilon$. We have $\mathbb{P}(|\varepsilon_{i}| \geq t) \leq C_1^\varepsilon \exp(-C_2^\varepsilon t)$. Hence we set
\begin{align*}
t =C^{\varepsilon}_2 \sigma - \log(n C^{\varepsilon}_1 ) \,\, &\Rightarrow \sigma = \frac{t+\log(n C_1^\varepsilon)}{C^{\varepsilon}_2}
\end{align*}
Thus,
\[
\alpha^u_{n,p} = \frac{\sqrt{2\{t+\log(4p)\}\{1+ (\frac{t+\log(n C_1^\varepsilon)}{C^{\varepsilon}_2})\}^2\}}}{\sqrt{n}}
\]
with probability at least $1-2\exp(-t)$ we have:
\begin{align*}
\frac{\sum_{i} Y_i  X_{ij^*} X_{ik^*} } { \| \mb Y \|_1} &\notin  \Big [-\frac{ m_2-\alpha^u_{n,p}}{m_1 +  m_\varepsilon +\alpha^u_{n,p}},\frac{ m_2-\alpha^u_{n,p}}{m_1 +  m_\varepsilon +\alpha^u_{n,p}} \Big] \\
\frac{\sum_{i=1}^n Y_i X_{ij} X_{ik}}{\| \mb Y \|_1}  &\in \Big [ - \frac{ m_2 (1- r_u) +\alpha^u_{n,p}}{  m_1 -\alpha^u_{n,p}}, \frac{ m_2 (1- r_u) +\alpha^u_{n,p}}{ m_1 -\alpha^u_{n,p}} \Big ] \,\, \forall (j,k) \neq (j^*,k^*).
\end{align*}
\end{proof}
Next we prove the equivalent result for the sign transform. The proof is very similar to the unbiased case:
%Consider the setup of Theorem~\ref{thm:unbiased}.
%Then there exists constants $ C_1^\varepsilon, C_2^\varepsilon > 0$ such that defining
%\[
%\alpha^s_{n,p} = \alpha^s_{n,p}(t) = \frac{\sqrt{2\{t+\log(4p)\}\{1+ (\frac{t+\log(n C_1^\varepsilon)}{C^{\varepsilon}_2})\}^2}}{\sqrt{n}},
%\]
\begin{lem}
\label{lem:signhelper}
Consider the setup of Theorem~\ref{thm:signdist}.
Then there exists constants $C_1^X, C_2^X, C_1^\varepsilon, C_2^\varepsilon > 0$ such that defining
\[
\alpha^s_{n,p}= \alpha^s_{n,p}(t)=\frac{\sqrt{2(t+\log(4p)) \Big( \Big(\frac{t+\log(p n C^{ X}_1)}{C^{ X}_2}\Big)^4+ \Big(\frac{t+\log(n C_1^\varepsilon)}{C^{\varepsilon}_2}\Big)^2 \Big)}}{\sqrt{n}},
\]
with probability at least $1-3\exp(-t)$ we have:
\begin{align*}
\frac{\sum_{i=1}^n Y_i \sgn( X_{ij^*} X_{ik^*} ) } { \| \mb Y \|_1} &\notin  \Big [-\frac{ m_1-\alpha_{n,p}^s}{m_1 +  m_\varepsilon +\alpha_{n,p}^s},\frac{ m_1-\alpha_{n,p}^s}{m_1+  m_\varepsilon +\alpha_{n,p}^s} \Big] \\
\frac{\sum_{i=1}^n Y_i \sgn(X_{ij} X_{ik})}{\| \mb Y \|_1}  &\in \Big [ - \frac{ m_1(1- r_s) +\alpha_{n,p}^s}{  m_1-\alpha_{n,p}^s}, \frac{ m_1(1- r_s) +\alpha_{n,p}^s}{  m_1-\alpha_{n,p}^s} \Big ] \,\,\, \forall \, (m,o) \neq (j^*,k^*).
\end{align*}
\end{lem}

\begin{proof}
First consider capped versions of the random variables of interest:
\begin{equation*}
X^{\prime}_{ij} = \begin{cases}
X_{ij} \textrm{ if } |X_{ij}| \leq M\\
M \sgn(X_{ij}) \textrm{ otherwise}
\end{cases} \,\,\,\, \textrm{ and } \,\,
\varepsilon^{\prime}_{i} = \begin{cases}
\varepsilon_{i} \textrm{ if } |\varepsilon_{i}| \leq \sigma\\
\sigma \sgn(\varepsilon_i) \textrm{ otherwise}
\end{cases}
\end{equation*}
where $M$ and $\sigma$ are to be chosen later.
Given these capped variables we can use Hoeffding's inequality as we now deal with bounded variables. We have to bound two terms:
\[
\frac{\sum_{i=1}^n Y_i \sgn(X^\prime_{ij^*} X^\prime_{ik^*})}{ \| \mb Y \|_1 } \textrm{ from below and } \frac{\sum_{i=1}^n Y_i \sgn(X^\prime_{ij} X^\prime_{ik})}{ \| \mb Y \|_1 } \textrm{ from above, for } (j,k) \neq (j^*,k^*)
\]
As in Lemma \ref{lem:unbiasedhelper} equation (\ref{ineq:useful}):
\[
A+B=\sum_{i=1}^n |X^\prime_{ij^*} X^\prime_{ik^*}| +\varepsilon^\prime_{i} \sgn(X^\prime_{ij^*} X^\prime_{ik^*}) \,\, \textrm{ and } \,\, C+D = \sum_{i=1}^n |X^\prime_{ij^*}X^\prime_{ik^*}+\varepsilon^\prime_i |.
\]
We deal with each term individually. Using Hoeffding's inequality we get:
\begin{itemize}
\item[$A: $] $\mathbb{P} \Big( \sum_{i=1}^p | X^\prime_{ij^*}X^\prime_{ik^*} | \leq  n m_1-\delta \Big) \leq \exp( - \delta^2 /  2n M^4)
)$
\item[$B: $] $\mathbb{P} \Big( \sum_{i=1}^n \varepsilon^\prime_i \leq -\kappa \Big) \leq \exp( - \kappa^2 /  2 n \sigma^2)
$
\item[$C: $] $\mathbb{P} \Big( \sum_{i=1}^n | X^\prime_{ij^*}X^\prime_{ik^*} | \geq  n m_1+\delta \Big) \leq \exp( - \delta^2 /  2n M^4) $
\item[$D: $] $\mathbb{P} \Big( \sum_{i=1}^n |\varepsilon^\prime| \geq n m_\varepsilon^\prime+ \kappa \Big)\leq \exp( -  2\kappa^2 /  n \sigma^2)$
\end{itemize}
This gives us a bound of the interaction strength of the true interaction pair:
\begin{align*}
\mathbb{P} \Big(\frac{\sum_{i} Y_i \sgn( X^\prime_{ij^*} X^\prime_{ik^*} ) } { \| \mb Y \|_1} &\geq \frac{n m_1-\delta-\kappa}{nm_1 + n m_\varepsilon +\delta+\kappa}  \Big)\\
&\geq 1- 2\exp( - \delta^2 /  2n M^4) -  2\exp( - \kappa^2 /  2 n \sigma^2)
\end{align*}
Similarly we can treat the interaction strength of the non interacting pairs:
\begin{itemize}
\item[$A: $] Here we use assumption $(C1)$. It implies
\[ r_s/2 \leq \mathbb{P}(\sgn(X^\prime_{ij^*} X^\prime_{ik^*})=\sgn(X^\prime_{ij} X^\prime_{ik})| \mb X) \leq 1-r_s/2.
\]
This we use for computing the expectation:
\begin{align*}
\mathbb{E}[X^\prime_{ij^*} X^\prime_{ik^*} \sgn(X^\prime_{ij} X^\prime_{ik}) ] &= \mathbb{E}[\mathbb{E}[ |X^\prime_{ij^*} X^\prime_{ik^*}| \sgn(X^\prime_{ij} X^\prime_{ik}X^\prime_{ij^*} X^\prime_{ik^*})]\\
&= \mathbb{E}[\mathbb{E}[ 2 |X^\prime_{ij^*} X^\prime_{ik^*}| \textbf{1}_{\{  \sgn(X^\prime_{ij} X^\prime_{ik}X^\prime_{ij^*} X^\prime_{ik^*}) = 1 \}}|\mb X]] - \mathbb{E}[ |X^\prime_{ij^*} X^\prime_{ik^*}| ]\\
&=\mathbb{E}[ \mathbb{E}[ 2 |X^\prime_{ij^*} X^\prime_{ik^*}| |\mb X]] \mathbb{P}(\sgn(X^\prime_{ij} X^\prime_{ik}X^\prime_{ij^*} X^\prime_{ik^*}) = 1 |\mb X) - \mathbb{E}[ |X^\prime_{ij^*} X^\prime_{ik^*}| ] \\
&= \mathbb{E}[ |X^\prime_{ij^*} X^\prime_{ik^*}| ] (2 \mathbb{P}(\sgn(X^\prime_{ij} X^\prime_{ik}X^\prime_{ij^*} X^\prime_{ik^*}) = 1 |\mb X ) - 1).
\end{align*}
Thus the expectation is given as:
\[
m_1( r_s-1) \leq E[X^\prime_{ij^*} X^\prime_{ik^*} \sgn(X^\prime_{ij} X^\prime_{ik}) ] \leq m_1(1- r_s).
\]
Hence, $\mathbb{P} \Big( \sum_{i=1}^n X^\prime_{ij^*} X^\prime_{ik^*} \sgn(X^\prime_{ij} X^\prime_{ik}) \geq n m_1(1-   r_s) + \delta \Big)  \geq \exp( -  2 \delta /  n M^4).
$
\end{itemize}
For the rest we use the same bounds as before (using $|X^\prime_{ij^*} X^\prime_{ik^*}+\varepsilon^\prime_i | \geq |X^\prime_{ij^*} X^\prime_{ik^*}|+\varepsilon^\prime_i$). This yields the bound
\begin{align*}
\mathbb{P} \Big( \frac{\sum_{i=1}^n Y_i \sgn(X^\prime_{ij} X^\prime_{ik})}{\| \mb Y \|_1}  &\leq \frac{n m_1(1- r_s) +\delta + \kappa}{ n m_1-\delta -\kappa} \Big) \\
&\geq 1 - \exp( -  2\delta^2 /  n M^4) - \exp( -  2\kappa^2 /  n \sigma^2).
\end{align*}
The above inequality needs to hold for the at most $p^2$ pairs that are not interactions, so that we effectively multiply the exponential terms with $p^2$. Another factor of $2$ is multiplied in for the negative sign, as the fraction also has to be bounded away from $-1$. In total we thus have:
\begin{align*}
\frac{\sum_{i} Y_i \sgn( X^\prime_{ij^*} X^\prime_{ik^*} ) } { \| \mb Y \|_1} &\notin  \Big [-\frac{n m_1-\delta-\kappa}{nm_1+ n m_\varepsilon +\delta+\kappa},\frac{n m_1-\delta-\kappa}{nm_1+ n m_\varepsilon +\delta+\kappa} \Big] \\
\frac{\sum_{i=1}^n Y_i \sgn(X^\prime_{ij} X^\prime_{ik})}{\| \mb Y \|_1}  &\in \Big [ - \frac{n m_1(1- r_s) +\delta + \kappa}{ n m_1-\delta -\kappa}, \frac{n m_1(1- r_s) +\delta + \kappa}{ n m_1-\delta -\kappa} \Big ] \,\, \forall (j,k) \neq (j^*,k^*)\\
\textrm{ with probability at least }& 1- 2 p \exp( - \delta^2 /  2n M^4) -  2 p\exp( - \kappa^2 /  2 n \sigma^2).
\end{align*}
Finally we have to set $\delta$ and $\kappa$ so that the probability is bigger than $1-\exp(-t)$.   This gives:
\[
\exp(-t) = 4 p \exp(-\delta^2/2nM^4) \textrm{ and } \exp(-t) = 4 p \exp(-\kappa^2/2n \sigma^2)
\]
This gives
\[
\delta = \sqrt{2n M^4(t+\log(4p))} \textrm{ and } \kappa = \sqrt{2n \sigma^2 (t+\log(4p))}
\]
Thus for $\alpha_{n,p}^s = \frac{\sqrt{2(t+\log(4p))(M^4+\sigma^2)}}{\sqrt{n}}$
\begin{align*}
\frac{\sum_{i} Y_i \sgn( X^\prime_{ij^*} X^\prime_{ik^*} ) } { \| \mb Y \|_1} &\notin  \Big [-\frac{ m_1-\alpha_{n,p}^s}{m_1+  m_\varepsilon +\alpha_{n,p}^s},\frac{ m_1-\alpha_{n,p}^s}{m_1+  m_\varepsilon +\alpha_{n,p}^s} \Big] \\
\frac{\sum_{i=1}^n Y_i \sgn(X^\prime_{ij} X^\prime_{ik})}{\| \mb Y \|_1}  &\in \Big [ - \frac{ m_1(1- r_s) +\alpha_{n,p}^s}{  m_1-\alpha_{n,p}^s}, \frac{ m_1(1-r_s) +\alpha_{n,p}^s}{  m_1-\alpha_{n,p}^s} \Big ] \,\, \forall (j,k) \neq (j^*,k^*)\\
\textrm{ with probability at least }& 1- \exp(-t).
\end{align*}
We now extend this result to the case of unbounded variables, that is we now assume that with high probability the variables $X_{ij}$ and $\varepsilon_i$ are bounded:
\begin{equation*}
\mathbb{P}( \, X_{ij} = X^{\prime}_{ij}, \,\, \forall \, i,j ) = 1-\exp(-t) \,\, \textrm{ and } \,\, \mathbb{P}( \varepsilon_i = \varepsilon^\prime_{ij}, \,\, \forall \, i  )=1-\exp(-t).
\end{equation*}
Here we used the sub-exponential tail behaviour of the $X_{ij}$ and $\varepsilon_i$. There exists constants $C_1^X$, $C_2^X$ such that $\pr(|X_{ij}| \geq t) \leq C_1^X \exp(-C_2^X t)$ and similarly for $\varepsilon$. Hence we set
\begin{align*}
t = C^{X}_2 M - \log(p n C^{ X}_1) \,\, &\Rightarrow M = \frac{t+\log(p n C^{X}_1)}{C^{ X}_2}\,\,\\
t =C^{\varepsilon}_2 \sigma - \log(n C^{\varepsilon}_1 ) \,\, &\Rightarrow \sigma = \frac{t+\log(n C_1^\varepsilon)}{C^{\varepsilon}_2}
\end{align*}
Thus we have
\[
\alpha^s_{n,p}=\frac{\sqrt{2(t+\log(4p)) \Big( \Big(\frac{t+\log(p n C^{ X}_1)}{C^{ X}_2}\Big)^4+ \Big(\frac{t+\log(n C_1^\varepsilon)}{C^{\varepsilon}_2}\Big)^2 \Big)}}{\sqrt{n}}
\]
\end{proof}
Next we prove \textbf{Theorem}~\ref{thm:unbiased}:
\begin{proof}
Given $\delta, \epsilon > 0$, choose $t$ such that $3\exp(-t) < \epsilon$. From (B3) we have that $\alpha^u_{n,p}(t)$ defined in Lemma~\ref{lem:unbiasedhelper} satisfies $\alpha^u_{n,p}(t) \to 0$ as $n \to \infty$. Thus from Lemma~\ref{lem:unbiasedhelper} we know that there exists $N$ such that for all $n \geq N$, with probability $1-\epsilon$ we have
\[
\frac{\log(\gamma_{j^* k^*}^g)}{\log(\gamma_{jk}^g)} < \frac{\log\{(1 + \tfrac{m_2}{m_1+m_\varepsilon})/2\}}{ \log\{(1 + \tfrac{m_1}{m_2(1-r_s)})/2\}} + \delta/2.
\]
Thus for $n \geq N$, applying Corollary~\ref{cor:generalization} we have that with probability $1-\epsilon$,
\[
C(M,L) \leq cn p^{1+\delta/2 + \frac{\log(1/2+m_2/2( (m_1  +m_\varepsilon)))}{\log(1/2+m_2(1-r_u)/(2m_1) )}},
\]
for some constant $c$.
\end{proof}
The proof of \textbf{Theorem}~\ref{thm:signdist} is very similar and is thus omitted.
%proceeds in a similar fashion:
%\begin{proof}
%Choose $t > 0$, then use Lemma \ref{lem:signhelper}: For $n,p$ large enough we have with probability $1-3\exp(-t)$ that there is a gap between $\gamma_{j^*k^*}^g$ and $\gamma_{mo}^g$:
%\begin{align*}
%\gamma_{j^* k^*}^g &\geq \frac{m_1-\alpha_{n,p}^s}{m_1+m_\varepsilon + \alpha_{n,p}^s}\\
%\gamma_{mo}^g &\leq \frac{m_1(1-2r_s)+\alpha_{n,p}^s}{m_1-\alpha_{n,p}^s},
%\end{align*}
%with $\alpha_{n,p}^s$ converging to zero (due to (C3)). Hence we can apply Corollary \ref{cor:generalization} to get the run time statement: Choose $t > 0$, then with probability $1-3\exp(-t)$ we have
%\[
%C(M,L) \leq \min \Big (C^{n,p}_1 np^2, C^{n,p,t}_2 n p^{1+\frac{\log(1/2+m_1/(2m_1+m_\varepsilon))}{\log((1-r_s))}} \Big),
%\]
%where the constant $C^{n,p}_2$ encodes the fact that Lemma~\ref{lem:signhelper} does not apply for small values of $n$ and $p$. It holds that $C^{n,p,t}_2 \to 1$ as $n,p$ increase. Define
%\[
%C_3 = \max_{n,p} \min \Big ( C^{n,p}_1 np^{1-\frac{\log(1/2+m_1/2(m_1+m_\varepsilon))}{\log((1-r_s))}}, C^{n,p,t}_2  \Big),
%\]
%then we have
%\[
%\mathbb{P}\Big( C(M,L) > C_3  n p^{1+\frac{\log(1/2+m_1/2(m_1+m_\varepsilon))}{\log((1-r_s))}} \Big) \leq 3 \exp(-t)
%\]
%and the statement follows.
%\end{proof}
%\newpage
\bibliographystyle{abbrvnat}
\bibliography{refs}